\documentclass[10pt,fullpage,letterpaper]{article}

\usepackage[utf8]{inputenc} 
\usepackage[T1]{fontenc}    
\usepackage{hyperref}       
\usepackage{url}            
\usepackage{booktabs}       
\usepackage{amsfonts}       
\usepackage{nicefrac}       
\usepackage{microtype}      
\usepackage{xcolor}         
\usepackage{comment}

\usepackage{amsmath}
\usepackage{natbib}
\usepackage[linesnumbered,ruled,vlined]{algorithm2e}
\usepackage{amsthm}
\usepackage{amssymb}
\usepackage{graphicx}
\usepackage{subfigure}[htbp]
\usepackage{xcolor,colortbl}
\definecolor{Gray}{gray}{0.85}
\definecolor{LightCyan}{rgb}{0.7,1,1}
\definecolor{LightRed}{rgb}{1,0.7,0.7}
\newcolumntype{a}{>{\columncolor{Gray}}c}

\usepackage{hhline}

\newtheorem{theorem}{Theorem}[section]
\newtheorem{lemma}{Lemma}[section]
\newtheorem{corollary}{Corollary}[section]
\newtheorem{definition}{Definition}[section]
\newtheorem{assumption}{Assumption}[section]
\newcommand{\Var}{{\bf Var}}
\newcommand{\simon}[1]{{\color{cyan} [Simon: #1]}}

\newcommand{\add}[1]{{\color{black} #1}}

\newcommand{\modify}[1]{{\color{black} #1}}

\newcommand{\uv}{\underline{\mathbf{V}}}
\newcommand{\uq}{\underline{\mathbf{Q}}}
\renewcommand{\S}{\mathcal{S}}
\newcommand{\A}{\mathcal{A}}
\newcommand{\up}{\underline{\pi}}
\newcommand{\gap}{\mathrm{gap}_\mathrm{min}}

\title{On Gap-dependent Bounds for Offline Reinforcement Learning}

\author{Xinqi Wang\footnote{\textbf{Tsinghua University. Email: \url{wangxqkaxdd@gmail.com}}} \and Qiwen Cui\footnote{\textbf{University of Washington. Email: \url{qwcui@cs.washington.edu}}}  \and Simon S. Du\footnote{\textbf{University of Washington. Email: \url{ssdu@cs.washington.edu}}} }

\begin{document}
\maketitle
\begin{abstract}
This paper presents a systematic study on gap-dependent sample complexity in offline reinforcement learning.
Prior work showed when the density ratio between an optimal policy and the behavior policy is upper bounded (the optimal policy coverage assumption), then the agent can achieve an $O\left(\frac{1}{\epsilon^2}\right)$ rate, which is also minimax optimal.
We show under the optimal policy coverage assumption, the rate can be improved to $O\left(\frac{1}{\epsilon}\right)$ when there is a positive sub-optimality gap in the optimal $Q$-function.
Furthermore, we show when the visitation probabilities of the behavior policy are uniformly lower bounded for states where an optimal policy's visitation probabilities are positive (the uniform optimal policy coverage assumption), 
the sample complexity of identifying an optimal policy is independent of $\frac{1}{\epsilon}$.
Lastly, we present nearly-matching lower bounds to complement our gap-dependent upper bounds.
\end{abstract}

\section{Introduction}
\label{sec:intro}

Reinforcement Learning~(RL) aims to learn a policy that maximizes the long-term reward in unknown environments \citep{sutton2018reinforcement}.
The success of reinforcement learning often relies on being able to deploy the algorithms that directly interact with the environment.
However, such direct interactions with real environments can be expensive or even impossible in many real-world applications, e.g., health and medicine~\citep{murphy2001marginal,gottesman2019guidelines}, education~\citep{mandel2014offline}, conversational AI~\citep{ghandeharioun2019approximating} and recommendation systems~\citep{chen2019top}. Instead, we have access to a dataset generated from some past suboptimal policies. 
Offline reinforcement learning (offline RL) aims to find a near-optimal policy using the offline dataset, and has achieved promising empirical successes~\citep{lange2012batch,levine2020offline}.

Recently, a line of works showed that under the single policy coverage assumption (Assumption~\ref{assumption:relative optimal policy coverage}), one can obtain a near-optimal policy with polynomial number of samples \citep{rashidinejad2021bridging(C*),LCBVI2021,towardsinstanceoptimal2021,settling2022}.
In particular, for the tabular setting, recent works have obtained minimax optimal sample complexity bounds $\widetilde{\Theta}\left(\frac{H^3SC^*}{\epsilon^2}\right)$ where $H$ is the planning horizon, $S$ is the number of states, $C^*$ is the constant for the single policy coverage assumption, and $\epsilon$ is the target accuracy~\citep{LCBVI2021,settling2022}.
The $O\left(1/\epsilon^2\right)$ is for the worst case and in many benign settings, one may use much fewer samples to learn a  (near-)optimal policy.

How the benign problem structures help reduce the sample complexity  has been extensively studied in the online bandits and reinforcement learning~\citep{clip2019,jonsson2020planning,wagenmaker2021task,xuhaike2021}.
In particular, when there is a suboptimality gap between the optimal policy and the rest, then one can obtain $\log T$-type regret in contrast to $\sqrt{T}$-type regret bounds in the worst case where $T$ is the number of interactions.
However, to our knowledge, how the gap structure helps reduce sample complexity in offline RL has not been thoroughly investigated.
This paper presents a systematic study on gap-dependent bounds for offline RL in the canonical tabular setting, with nearly-matching upper and lower bounds in different regimes.

\subsection{Main Contributions}
We present novel analyses for the standard VI-LCB algorithm (Algorithm~\ref{alg:LCB with subsampling}).
Our main results are summarized in Table~\ref{table:bound table}.

\begin{enumerate}
    \item We develop a novel technique, \emph{deficit thresholding}, to obtain gap-dependent bounds in offline RL. Different from the clip trick widely used in online RL~\citep{clip2019,xuhaike2021} to obtain gap-dependent bounds, our deficit thresholding technique is adaptive to the problem-instance. As will be shown in Section \ref{subsection:clip tech}, this technique helps reduce the dependency on $H$ by adapting to the variance of the estimate. 

\item Using the deficit thresholding technique, we obtain the first gap-dependent bound under the optimal policy coverage assumption. Specifically, we obtain an $\widetilde{O}\left(\frac{H^4SC^*}{\epsilon\gap}\right)$ bound where $\gap$ is the minimum suboptimality gap between the optimal $Q$-value of the best action and that of the second-to-the-best action.\footnote{The gap-dependent bounds in the online setting also depend on the suboptimality gaps of the actions other than the second-to-the-best action.
 This dependency is \emph{not needed} for offline RL because even if the dataset does not have any information about the sub-optimal actions, the agent can still learn a near-optimal policy (as long as the dataset covers an optimal action). On the other hand, this dependency is needed in the online setting because the agent needs to explore the all actions.  }
 Notably, compared with the worst-case gap-indepdent, the rate improves from $1/\epsilon^2$ to $1/\epsilon$.

 \item We also present the first gap-dependent lower bound for offline RL, $\Omega\left(\frac{H^2SC^*}{\epsilon\gap}\right)$ to show the $O\left(1/\epsilon\right)$ is unimprovable even with the gap condition, and our upper bound is tight up to an $H^2$ factor. The main insight from our lower bound is that if there exists a state whose visitation probability of the behavior policy and the optimal policy is $O\left(\epsilon\right)$, then we cannot learn much of the state using the offline dataset and will inevitably incur an $O\left(\epsilon\right)$ error.
\

\item We further study what condition permits even faster rate than $O\left(1/\epsilon\right)$. Leveraging our lower bound mentioned above, we propose a new condition, \emph{uniform optimal policy coverage}, which posits that the visitation probabilities of the behavior policy are uniformly lower bounded by $P$ for states where an optimal policy's visitation probabilities are positive. Under this assumption, we obtain an $\widetilde{O}\left(\frac{H^3}{P\gap^2}\right)$ bound. Importantly, this bound is \emph{independent of $1/\epsilon$} (not even $\log(1/\epsilon)$), a.k.a., one can identify an \emph{exact optimal} policy.
We also complement this upper bound with an $\Omega\left(\frac{H}{P\gap^2}\right)$ lower bound to show our upper bound is tight up to an $H^2$ factor.

\end{enumerate}

Lastly, we note that all of our bounds are obtained with the same algorithm, i.e., the algorithm automatically exploits the benign problem structure without any prior knowledge.

\begin{table*}[!t]
    \label{table:bound table}
\centering\resizebox{0.7\columnwidth}{!}{    
    \begin{tabular}{|c|c|c|}
        \hline
        \rowcolor{white}
      Condition & Upper bound & Lower bound \\
				\hhline{|=|=|=|}
				$C^*$   & $\widetilde{O}\left(\frac{H^3SC^*\modify{\log\frac{1}{\delta}}}{\epsilon^2}\right)$  & $\Omega\left(\frac{H^3SC^*}{\epsilon^2}\right)$ \\
				\hline
        $C^*$, $\gap$  & 	\cellcolor{Gray} $\widetilde{O}\left(\frac{H^4SC^*\modify{\log\frac{1}{\delta}}}{\epsilon\gap}\right)$  & 	\cellcolor{Gray} $\Omega\left(\frac{H^2SC^*}{\epsilon\gap}\right)$ \\
        \hline

        $P$, $\gap$   & 	\cellcolor{Gray} $\widetilde{O}\left(\frac{H^3\modify{\log\frac{1}{\delta}}}{P\mathrm{gap}^2_{\min}}\right)$ & 	\cellcolor{Gray} $\Omega\left(\frac{H}{P\mathrm{gap}^2_{\min}}\right)$   \\
        \hline
    \end{tabular}
}
\caption{Sample Complexity Bounds for different conditions about the sub-optimality and coverage. Cells in gray are the contributions of this work. 
The results in the first line without suboptimality gap assumptions were obtained in \cite{LCBVI2021,settling2022}.
$C^*$ stands for the relative optimal policy coverage coefficient, i.e., $\underset{h,s}{\max}\frac{d_h^{*}(s)}{d_h^{\mu}(s)}$ where $d^*_h(s)$ is the visitation probability of the optimal policy for state $s$ at level $h$ and $d_h^{\mu}(s)$ is the visitation probability of the behavior policy $\mu$ for state $s$ at level $h$. $P$ stands for the uniform optimal policy coverage coefficient, i.e.,
$\underset{h,s\mid d_h^*(s)>0}{\min}d_h^{\mu}(s)$. $\gap$ is the minimum non-zero suboptimality gap among all time-state-action tuples, i.e., $\underset{h,s,a\mid a\text{ is not optimal}}{\min}\mathbf{V}^*_h(s)-\mathbf{Q}^*_h(s,a)$.
    }
\end{table*}

\section{Related Work}
We focus on existing theoretical results on gap-dependent bounds and offline RL in the tabular setting.

\paragraph*{Theoretical Results on Offline Tabular RL.}
Theoretical analysis of offline RL can be traced back to \citet{szepesvari2005finite(sup1)}, under the uniform coverage assumption where \emph{every state-action} pairs are visted by the behavior policy with a positive probability.
Sharp sample complexity bounds have been obtained under this assumption \citep{xie2021batch(off4),xie2019towards,yin2020near,yin2021near,nearlyhorizonfree2021(strongeuniformcoverage)}. 
Recently, a line of works showed that  under a much weaker assumption, single policy coverage, one can design sample efficient algorithms with both model-based and model-free methods based on the pessimism principle,  \citep{rashidinejad2021bridging(C*),towardsinstanceoptimal2021,xie2021batch(off4), jin2021pessimism,uehara2021pessimistic,uehara2021representation,zanette2021provable}.
Recently, \cite{towardsinstanceoptimal2021} obtained a problem-dependent sample complexity in terms of the variance.

\paragraph*{Instance-dependent Sample Complexity in Online Learning and Generative Models.} In online RL, a line of work studied the upper and lower bounds of gap-dependent sample complexity in both the regret and PAC settings~\citep{clip2019,xuhaike2021,dann2021beyond,jonsson2020planning,wagenmaker2021task,MOCA2021,tirinzoni2021fully}.  Besides gap-dependent bounds, there are other problem-dependent bounds such as first-order and variance-dependent bounds~\cite{zanette2019tighter}. For generative models, \citet{zanette2019almost} derived an upper bound depending on both variance and gap information.

\paragraph*{Gap-dependent bounds in Offline Learning.}
The most related work is by \cite{fastrate2021} who
studied the convergence rate of $Q$-learning in the discounted  MDPs under the uniform coverage assumption. They proved that one can obtain an $\epsilon$-optimal policy with $O\left(S^3A^3\log(1/\epsilon)/(1-\gamma)^4P^2\gap^2\right)$ samples for tabular MDP and $\widetilde{O}(1/\epsilon)$ (ignoring other parameters) for linear MDP. In comparison, we show that under the weaker uniform optimal policy coverage assumption, we can identify an exact optimal policy with $O\left(H^3/P\gap^2\right)$ sample complexity, which has no dependency on $1/\epsilon$.

\section{Preliminaries}
\paragraph*{Notations.} We let $[n]=\{1,2\cdots n\}$. For two vectors $a,b$ of the same length $k$, we use $a\circ b$ to denote the 
Hadamard product $(a_1b_1,a_2b_2\cdots a_kb_k)$. We use the standard definitions of $O(\cdot),\Theta(\cdot),\Omega(\cdot)$ to hide absolute constants, and tilded notations $\widetilde{O}(\cdot),\widetilde{\Theta}(\cdot),\widetilde{\Omega}(\cdot)$ to hide absolute constants as well as poly-logarithmic factors \modify{except for $\log\frac{1}{\delta}$}. $\Var_{p}(V)=p^\top V\circ V-(p^\top V)^2$ refers to the variance of $V$ with respect to the weight $p$. $a\lesssim b$ 
means $a\leq Cb$ for some positive absolute constant $C$, and similarly $a\gtrsim b$ means $a\geq Cb$. $\mathbb{I}\{\xi\}$ is the indicator function of the event $\xi$, which equals 1 when $\xi$ is true and 0 otherwise. $\Delta(\mathcal{X})$ is the probability simplex over $\mathcal{X}$.

\subsection{Markov Decision Processes}

We consider tabular finite-horizon time-inhomogeneous MDPs described by the tuple $\mathcal{M}=(\mathcal{S}, \mathcal{A}, H, \mathcal{P},p_0, r)$.
Here $\mathcal{S}$ is the state space with cardinality $S$, $\mathcal{A}$ is the action space with cardinality $A$, $\mathcal{P}=\{p_{h,s,a}\}_{(h,s,a)\in[H]\times\mathcal{S}\times\mathcal{A}}$ with $p_{h,s,a}\in\Delta(\mathcal{S})$
is the transition kernel at timestep $h$, state $s$ and action $a$, $p_0\in\Delta(\mathcal{S})$ is the initial state distribution of $s_1$, and $r=\{r_1,r_2,\dots,r_H\}$ with $r_h:\mathcal{S}\times\mathcal{A}\rightarrow [0,1]$ is the reward function.\footnote{We assume a known deterministic reward function as the main difficulty lies in learning the transition probability. All the conclusions in this paper can be proved for MDPs with unknown 1-subGuassian reward.} For each episode, the player will generate a trajectory $\{(s_h,a_h,r_h)\}_{h=1}^H$ where $s_1\sim p_0$, $s_{h+1}\sim p_{h,s_{h},a_h}$, and $r_h=r_h(s_h,a_h)$, by controlling the actions $\{a_h\}_{h=1}^H$. The target of the player is to maximize the total reward $\sum_{h=1}^Hr_h$. 

\paragraph*{Policies.} A policy $\pi=\{\pi_h(\cdot| s)\}_{(h,s)\in[H]\times \mathcal{S}}$ refers to a set of distributions over $\mathcal{A}$. With a slight abuse of the notations, when a policy $\pi$ is deterministic, we use $\pi_h(s)$ to denote the action taken at timestep $h$ and state $s$.
 We define $\mathbb{E}_{\pi,\mathcal{M}}[\cdot]\triangleq\mathbb{E}_{\phi\sim(\pi,\mathcal{M})}[\cdot]$, where $\phi=\{(s_h,a_h,r_h)\}_{h\in[H]}$ is a trajectory sampled using policy
$\pi$ in the MDP $\mathcal{M}$ and the expectation is over the randomness of both the policy and the transitions. $\mathcal{M}$ will be omitted when there is no confusion.

\paragraph*{Value Functions, Q-Functions and Policy Distributions.} For a given policy $\pi$ and an MDP $\mathcal{M}$, we define the state value function and the state-action value function to be 
\begin{align*}
\mathbf{V}^\pi_h(s)\triangleq\mathbb{E}_\pi \left[\sum_{h'=h}^Hr_{h'}(s_{h'},a_{h'})\mid s_h=s\right],\mathbf{Q}^\pi_h(s,a)\triangleq\mathbb{E}_\pi\left[\sum_{h'=h}^Hr_{h'}(s_{h'},a_{h'})\mid s_h=s, a_h=a\right].
\end{align*}
We define the optimal Q-function as $\mathbf{Q}_h^*(s,a)\triangleq\sup_\pi \mathbf{Q}_h^\pi(s,a)$, and similarly $\mathbf{V}_h^*(s)\triangleq\sup_\pi \mathbf{V}_h^\pi(s)$ for all $(h,s,a)\in[H]\times\S\times\A$. It is well known that there exists a deterministic optimal policy $\pi^*$ that can achieve the above maximum for all $s\in\mathcal{S}$, $a\in\mathcal{A}$ and $h\in[H]$ simultaneously.

We denote the value of a policy by  $\mathbf{V}_0^\pi=\mathbb{E}_\pi\left[\sum_{h=1}^Hr_{h}(s_{h},a_{h})\right]$, 
and the value of the optimal policy by $\mathbf{V}_0^*=\mathbf{V}_0^{\pi^*}$. A policy $\pi$ is $\epsilon$-optimal if $$\mathrm{Suboptimal}(\pi)\triangleq \mathbf{V}_0^*-\mathbf{V}_0^{\pi}\leq\epsilon.$$
We use $d_h^\pi(\cdot)$ to denote the probability of reaching a state $s$ under policy $\pi$:
\begin{align*}
    d_h^\pi(s) \triangleq \mathbb{E}_\pi[\mathbb{I}\{s_h=s\}],\ d_h^\pi(s,a) \triangleq \mathbb{E}_\pi[\mathbb{I}\{(s_h,a_h) = (s,a)\}]=d_h^\pi(s)\pi_{h}(a\mid s).
\end{align*}

\paragraph*{Sub-optimality Gap.} For a MDP instance $\mathcal{M}$, we define the gap at $(h,s,a)\in[H]\times\S\times\A$ to be $\mathrm{gap}_{h}(s,a)=\mathbf{V}^*_h(s) - \mathbf{Q}^*_h(s,a)$, which is always non-negative because of the 
definition of $\pi^*$. Our results will depend on the smallest positive gap:
$\gap \triangleq \underset{(h,s,a)\mid\mathrm{gap}_h(s,a)>0}{\min}\mathrm{gap}_{h}(s,a)$, which quantifies the difficulty of learning the optimal policy in the MDP instance.\footnote{We assume that at least one gap is positive. Otherwise all the actions are optimal and no learning is needed. }

\subsection{Offline Reinforcement Learning}
\label{subsection: offline rl}
\paragraph{Offline Learning.} For offline reinforcement learning, we want to solve an MDP $\mathcal{M} = (\S,\A,H,\mathcal{P},r)$ with unknown transitions by utilizing a given dataset collected by some unknown behavior policy $\mu$. Note that the algorithm is not allowed to perform any kind of additional sampling. The dataset is $\mathcal{D}=\{\{(s_{h,i},a_{h,i},r_{h,i})\}_{h\in[H]}\}_{i\in[N]}$, which contains $N$ trajectories collected by the behavior policy $\mu$ independently.
A $(\epsilon, \delta)$-PAC offline reinforcement learning algorithm is defined to output an $\epsilon$-optimal policy $\pi$ with probability at least $1-\delta$.

\paragraph{Assumptions.} We introduce two dataset assumptions that will be used in this paper. 

\begin{assumption}[Uniform optimal policy coverage]
    \label{assumption:uniform optimal policy coverage}
    We define the uniform optimal policy coverage coefficient to be
    \begin{equation*}
    P\triangleq \min_{\pi^*} \modify{\underset{h,s:  d_h^{\pi^*}(s,a)>0}{\min}d_h^\mu(s, a)},
    \end{equation*}
    where $\pi^*$ is an optimal policy. We assume that $P>0$. 
\end{assumption}

Assumption \ref{assumption:uniform optimal policy coverage} states that the behavior policy $\mu$ covers all the state-action pairs that some $\pi^*$ will choose with positive probability. This is a natural assumption if we want to recover the optimal policy.  

A closely related assumption is the uniform coverage assumption, i.e., all $(h,s,a)$ tuples are covered by the dataset \citep{asymptotic2020(strongeuniformcoverage),nearlyhorizonfree2021(strongeuniformcoverage),nearoptimalprovable2021(strongeuniformcoverage), fastrate2021}. Our assumption is significantly weaker as it only assumes covering the optimal policy. We will prove that under this assumption, we can identify the optimal policy with finite samples.
   
\begin{assumption}[Optimal policy coverage]
\label{assumption:relative optimal policy coverage}
We define the relative optimal policy coverage coefficient to be
\begin{equation*}
C^*\triangleq \max_{\pi^*}\underset{(h,s,a)\in[H]\times\mathcal{S}\times\mathcal{A}}{\max}\frac{d_h^{\pi^*}(s,a)}{d_h^\mu(s,a)}
\end{equation*}
with convention that $0/0=0$, where $\pi^*$ is an optimal policy. We assume that $C^*<\infty$.
\end{assumption}

Similar coverage assumption has been widely adopted in \citet{settling2022,pessimisticQ2022,pessimisticQpre2022,jin2021pessimism,rashidinejad2021bridging(C*)}. Researchers have designed algorithms based on the pessimism principle to efficiently solve offline RL problems under this assumption. Assumption \ref{assumption:relative optimal policy coverage} is usually weaker than 
Assumption \ref{assumption:uniform optimal policy coverage}. For example, if there exists one unique optimal policy $\pi^*$ and $\mu=\pi^*$, we have $C^*=1$ while $P$ can still be arbitrarily small. In addition, we always have $C^*\leq \frac{1}{P}$.

\IncMargin{1em}
\begin{algorithm}[!t]
    \DontPrintSemicolon
    \SetKwData{Left}{left}\SetKwData{This}{this}\SetKwData{Up}{up}
    \SetKwFunction{Union}{Union}\SetKwFunction{FindCompress}{FindCompress}
    \SetKwInOut{Input}{input}\SetKwInOut{Output}{output}

    \Input{Dataset $\mathcal{D}_0$, reward function $r$}
    \BlankLine
    Set $\uq_{H+1}(s,a)=0$\;
    Set $\uv_{H+1}(s,a)=0$\;
    \For{$h\leftarrow H$ \KwTo $1$}{
        Compute the empirical transition kernel $\hat{P}_h$\;
        $\hat{P}_{h,s,a}(s') = \frac{N_h(s,a,s')}{N_h(s,a)}$ with $0/0=0$\;
        \For{$s\in\mathcal{S}, a\in\mathcal{A}$}{
            $b_h(s,a)\leftarrow C_b\sqrt{\frac{\Var_{\hat{P}_{h,s,a}}(\uv_{h+1})\iota}{N_h'(s,a)}} + C_b\frac{H\iota}{N_h'(s,a)}$, where $N_h'(s,a) = N_h(s,a)\vee \iota $\;
            $\uq_h(s,a) \leftarrow \max\{0,r_h(s,a) + \hat{P}_{h,s,a}^\top\uv_{h+1} - b_h(s,a)\}$
        }
        \For{$s\in\mathcal{S}$}{
            $\uv_h(s)\leftarrow\underset{a\in\mathcal{A}}{\max}\ \uq_{h}(s,a)$\;
            $\up_h(s)\leftarrow \underset{a\in\mathcal{A}}{\arg\max}\ \uq_h(s,a)$ 
        }
    }
    \Output{policy $\up$}
\caption{VI-LCB}
\label{alg:LCB}
\end{algorithm}

\DecMargin{1em}
\IncMargin{1em}
\begin{algorithm}[!t]
    \DontPrintSemicolon
    \SetKwData{Left}{left}\SetKwData{This}{this}\SetKwData{Up}{up}
    \SetKwFunction{Union}{Union}\SetKwFunction{FindCompress}{FindCompress}
    \SetKwInOut{Input}{input}\SetKwInOut{Output}{output}

    \Input{Dataset $\mathcal{D}$, reward function $r$}
    \BlankLine
    Split $\mathcal{D}$ into 2 halves containing same number of sample trajectories, $\mathcal{D}^{\mathrm{main}}$ and $\mathcal{D}^{\mathrm{aux}}$\;
    $\mathcal{D}_0=\{\}$\;
    \For{$(h,s)\in[H]\times \mathcal{S}$}{
        $N^{\mathrm{trim}}_h(s)\leftarrow \max\{0, N^{\mathrm{aux}}_h(s) - 10\sqrt{N_h^{\mathrm{aux}}(s)\log\frac{HS}{\delta}}\}$\;
        Randomly subsample $\min\{N^{\mathrm{trim}}_h(s), N^{\mathrm{main}}_h(s)\}$ samples of transition from $(h,s)$ from $\mathcal{D}^{\mathrm{main}}$ to add 
        to $\mathcal{D}_0$
    }
    $\up\leftarrow \text{VI-LCB}(\mathcal{D}_0, r)$\;
    \Output{policy $\up$}
\caption{Subsampled VI-LCB}
\label{alg:LCB with subsampling}
\end{algorithm}
\DecMargin{1em}

\subsection{Subsampled VI-LCB}
We briefly introduce the algorithm (Algorithm \ref{alg:LCB with subsampling}) that will be used in the analysis. Value Iteration with Lower Confidence Bound (VI-LCB) was first introduced by \citet{rashidinejad2021bridging(C*)} and improved by \citet{settling2022}. The main idea is to maintain a pessimistic estimate on the value functions so that the suboptimality of the output policy only depends on the uncertainty of the optimal policy.
By utilizing the subsampling technique and Bernstein-style bonus, subsampled VI-LCB achieves the minimax sample complexity $\widetilde{O}(\frac{H^3SC^*}{\epsilon^2})$.

In the algorithm, $N_h(s,a)$ refers to the number of sample transitions starting from state $s$, taking action $a$ at time step $h$ of some given dataset, and $N_h(s)=\sum_{a\in\mathcal{A}}N_h(s,a)$. Superscripts stand for the dataset. See \citet{settling2022} for a more detailed description of the algorithm. 

\section{Finding an Exact Optimal Policy with Assumption~\ref{assumption:uniform optimal policy coverage}}
In this section, we show that we can identify the exact best policy with finite samples by utilizing the gap structure under Assumption \ref{assumption:uniform optimal policy coverage}. On the other hand, for the minimax sample complexity $\widetilde{O}(H^3SC^*\epsilon^{-2})$, it will become infinite when $\epsilon$ approaches 0. Note that directly setting $\epsilon<\mathrm{gap}_{\mathrm{min}}$ does not imply that the output policy $\pi$ is optimal. Instead, we only have 
$$\mathbf{V}^\pi_0\geq \mathbf{V}^*_0-\epsilon,$$
while $\pi$ can still be suboptimal at states visited with low probability.

\begin{theorem}
    \label{theorem:ub p}
    For an MDP $\mathcal{M}$ and a behavior policy $\mu$ with uniform optimal coverage coefficient $P$, if the number of sample trajectories satisfies
    \begin{equation*}
        N\geq \widetilde{O}\left(\frac{H^3\log\frac{1}{\delta}}{P\gap^2}\right),
    \end{equation*}
    then with probability at least $1-\delta$, Algorithm \ref{alg:LCB with subsampling} returns an optimal policy.
\end{theorem}

\begin{proof}[Sketch of Proof]
First, if for all $(h,s,a)\in[H]\times\mathcal{S}\times\mathcal{A}$ satisfying $d^*_h(s,a)>0$, we have
\begin{equation}\label{equ:uniform}
    \uq_h(s,a)> \mathbf{Q}^*_h(s,a)-\gap,
\end{equation}
then we have $\up$ is an optimal policy. This is because as $\uq$ is a pessimistic estimate of $Q^*$, we have
$$\uq_h(s,a)>\mathbf{Q}^*_h(s,a)-\gap\geq \mathbf{Q}_h^*(s,a')\geq\uq_h(s,a'),$$
for any action $a'$ that is sub-optimal. As a result, $\up$ chooses the optimal action for all $(h,s,a)\in[H]\times\mathcal{S}\times\mathcal{A}$ covered by the optimal policy, which implies $\up$ is an optimal policy.

Second, we show that \eqref{equ:uniform} is satisfied with $N\geq \widetilde{O}\left(\frac{H^4}{P\gap^2}\right)$. This is because if the number of samples at $(h,s,a)$ exceeds $\widetilde{O}({H^4}/\gap^{2})$, we can guarantee that the estimation error at that step is smaller than $\gap/H$ by Hoeffding's inequality. Then the accumulated estimation error at $\uq_h(s,a)$ can be bounded by $(H-h)\gap/H<\gap$, which implies \eqref{equ:uniform}. To further improve the dependence on $H$, we will use Bernstein's inequality and the proof is deferred to Appendix \ref{subsection: proof of ub p}.
\end{proof}

\section{Gap-dependent Upper Bounds with Assumption \ref{assumption:relative optimal policy coverage}}
In the previous section, we show that the optimal policy can be identified if Assumption \ref{assumption:uniform optimal policy coverage} is satisfied. However, the uniform optimal coverage coefficient $P$ can be very small, which makes the bound $N\geq \widetilde{O}\left(\frac{H^3}{P\gap^2}\right)$ useless. In this section, we present two results on learning an $\epsilon$-optimal policy with assumption (Assumption \ref{assumption:relative optimal policy coverage}) and we provide the proof sketch in the next section. \add{The full proof can be found in Appendix \ref{appendix:B}.}

\begin{theorem}
    \label{theorem:ub c}
    For an MDP $\mathcal{M}$ and behavior policy $\mu$ with relative optimal policy coverage coefficient $C^*$, if the number of sample trajectories satisfies
    \begin{equation*}
        N\geq \widetilde{O}\left( \frac{H^4SC^*\log\frac{1}{\delta}}{\epsilon\gap}\right),
    \end{equation*}
    Algorithm \ref{alg:LCB with subsampling} returns an $\epsilon$-suboptimal policy with 
    probability at least $1-\delta$.
\end{theorem}
Theorem \ref{theorem:ub c} shows that to learn an $\epsilon$-optimal policy, we only need $\widetilde{O}(1/\epsilon\gap)$ samples, which significantly improves the minimax sample complexity $\widetilde{O}(1/\epsilon^2)$ when $\epsilon\ll\gap$. 

Now we show that we can further improve the bound if both Assumption \ref{assumption:uniform optimal policy coverage} and Assumption \ref{assumption:relative optimal policy coverage} are satisfied. 

\begin{theorem}
    \label{theorem:ub pc}
    For an MDP $\mathcal{M}$ and behavior policy $\mu$ with uniform optimal policy coverage coefficient $P$ and relative optimal policy coverage coefficient $C^*$,  if the number of sample trajectories satisfies
    \begin{equation*}
        N\geq \widetilde{O} \left(\frac{H^3SC^*\log\frac{1}{\delta}}{\epsilon\gap}+\frac{H\log\frac{1}{\delta}}{P}\right),
    \end{equation*}
    Algorithm \ref{alg:LCB with subsampling} returns an $\epsilon$-suboptimal policy with 
    probability at least $1-\delta$.
\end{theorem}
Theorem \ref{theorem:ub pc} improves an $H$ factor compared with Theorem \ref{theorem:ub c}, with the cost of an extra $H/P$ term. As the additional term is a constant with respect to $\epsilon$ \add{and $\mathrm{gap}_{\min}$}, the bound is improved when $\epsilon\mathrm{gap}_{\min}$ is small. We believe this $H/P$ can be removed and we leave it to future works.

\section{Main Proof Techniques}
In this section, we will provide a proof sketch for Theorem \ref{theorem:ub c} and Theorem \ref{theorem:ub pc}. 
\subsection{Pessimistic Algorithms}
First, we define pessimistic algorithms and imaginary MDPs $\underline{\mathcal{M}}=(\mathcal{S},\mathcal{A},H,\mathcal{P},\underline{r})$ determined by the pessimistic algorithms. Our analysis will generally hold for all pessimistic algorithms defined by Definition \ref{definition: pessimistically estimated MDP} and we will show that VI-UCB (Algorithm \ref{alg:LCB with subsampling}) is a pessimistic algorithm. We use VI-LCB  only to derive the final sample complexity guarantees.  

\begin{definition}[Pessimistic algorithms]
    An offline learning algorithm with output policy $\up$ is pessimistic if with probability at least $1-\delta$, the following arguments hold,
    \begin{enumerate}
        \item It maintains a pessimistic estimate  $\uq$ of the true $Q^*$.
        \item $\uq$ is the optimal Q function of an imaginary MDP $\underline{\mathcal{M}}=(\mathcal{S},\mathcal{A},H,\mathcal{P},p_0,\underline{r})$, where $\underline{r}_h(s,a)\leq r_h(s,a)$ for all  $(h,s,a)\in[H]\times \mathcal{S}\times\mathcal{A}$.\footnote{Here we loosen the definition of MDP by allowing the reward function to have negative value. $\underline{r}$ maybe negative, but as will be shown in appendix, $\uv^{\up}$ and $\uq^{\up}$ are still non-negative, thus does not affect our analysis.} 
        \item $\up$ is the greedy policy with respect to $\uq$. 
    \end{enumerate}
\end{definition}
Most of the existing offline RL algorithms are pessimistic algorithms and we will also prove it for VI-LCB \citep{settling2022} in Appendix \ref{appendix:B}.

\begin{definition}[Deficit]
    For any pessimistic algorithm and the corresponding imaginary MDP $\underline{\mathcal{M}}=(\mathcal{S},\mathcal{A},H,\mathcal{P},p_0,\underline{r})$, we define the deficit to be
    \begin{equation*}
        \mathbf{E}_h(s,a) = r_h(s,a) - \underline{r}_h(s,a),\ \forall (h,s,a)\in[H]\times\S\times\A.
    \end{equation*}
\end{definition}
From the definition of pessimistic algorithms, we have $\mathbf{E}_h(s,a)\geq 0$ immediately. Intuitively, deficit stands for how pessimistic the estimates are. Note that the deficit is related to the algorithm itself and usually we can bound it by utilizing concentration inequalities. 

\subsection{Deficit Thresholding for Analysising LCB-style Algorithms}
\label{subsection:clip tech}

By defining $\uv^{\pi}$ to be the value function of policy $\pi$ in $\underline{\mathcal{M}}$ and recalling that $\up$ is the optimal policy in $\underline{\mathcal{M}}$, we have $\uv_0^*\leq \uv_0^{\up}\leq \mathbf{V}_0^{\up}\leq\mathbf{V}_0^*$  and thus
\begin{equation*}
    \mathbf{V}_0^* - \uv_0^* = \sum_{h=1}^H\mathbb{E}_{\pi^*}[r_h(s_h,a_h)-\underline{r}_h(s_h,a_h)] = \sum_{h=1}^H\mathbb{E}_{\pi^*}[\mathbf{E}_h(s_h,a_h)]
\end{equation*}
can upper bound the suboptimality of $\up$. Surprisingly, we will show that even if we threshold the deficit, a similar upper bound still holds.
Define the thresholded deficits and the corresponding reward functions as $$\ddot{\mathbf{E}}_h(s,a) \triangleq \max\{0, \mathbf{E}_h(s,a)-\epsilon_{h}(s,a)\},\ \ddot{r}_h(s,a) \triangleq r_h(s,a) - \ddot{\mathbf{E}}_h(s,a),\ \forall (h,s,a)\in[H]\times\S\times\A,$$  
for any non-negative threshold function $\epsilon_h(s,a)$. Then we can define a thresholded MDP $\ddot{\mathcal{M}}=(\mathcal{S},\mathcal{A},H,\mathcal{P},\ddot{r})$ and we use
$\ddot{\mathbf{V}}$ to denote the value function in $\ddot{\mathcal{M}}$. Now we present the key lemma showing that if the thresholding function satisfies certain conditions related to $\gap$, then the suboptimality of $\up$ can still be bounded by the thresholded deficit. 
\begin{lemma}
    \label{theorem:clip theorem(informal)}
    Suppose for a thresholding function $\epsilon_{h}(s,a)$, for all $(h,s)\in[H]\times\S$, we have
    \begin{equation}
        \label{requirement:theorem 1}
        \uv^*_h(s) + \frac{\gap}{2}\geq \ddot{\mathbf{V}}^*_h(s).
    \end{equation}
    Then we can bound the suboptimality of $\up$ by the thresholded deficit:
    \begin{equation*}
        \mathrm{suboptimal}(\up)=\mathbf{V}_0^* - \mathbf{V}_0^{\up}\leq \mathbf{V}_0^* - \uv_0^*\leq 2 (\mathbf{V}_0^* - \ddot{\mathbf{V}}_0^*) = 2\sum_{h=1}^H \mathbb{E}_{\pi^*}[\ddot{\mathbf{E}}_h(s_h,a_h)].
    \end{equation*}
\end{lemma}
\add{See Appendix \ref{ssection: main theorem} for the rigorous proof, where this lemma is decomposed and restated as Theorem \ref{theorem:main theorem} and Definition \ref{def: gap restriction event}.}
One way to satisfy \eqref{requirement:theorem 1} is to set $\epsilon_{h}(s,a)=\frac{\gap}{2H}$ for all $(h,s,a)\in[H]\times\S\times\A$. As a result, the reward at each timestep $h$ is increased by at most $\frac{\gap}{2H}$ after thresholding, so the overall increase of the value function is bounded by $\gap/2$. 
This kind of $\epsilon_{h}(s,a)$ leads to the following corollary. 
\begin{corollary}
    \label{corollary:upper bound 1}
    For any pessimistic algorithm, we have
    \begin{equation*}
        \mathbf{V}_0^* - \mathbf{V}_0^{\up}\leq 2\sum_{h=1}^H\mathbb{E}_{\pi^*}\left[\max\{0,\mathbf{E}_h(s_h,a_h) - \frac{\gap}{2H}\}\right].
    \end{equation*}
\end{corollary}

In addition, note that $\mathbf{E}_h(s_h,a_h)$ depends on variance of the transition $p_{h,s_h,a_h}$ as if the transition variance is small, the estimate will be accurate and the deficit will be small. This inspires us to threshold adaptively based on the variance. We design the following adaptive threshold function:
\begin{equation*}
    \epsilon_h(s,a) \propto \frac{\Var_{\hat{P}_{h,s,a}}(\uv_{h+1})}{H^2}\gap.
\end{equation*}
It turns out that $\widetilde{\Omega}(\frac{H}{P})$ sample complexity is enough to guarantee this kind of threshold function satisfies \eqref{requirement:theorem 1}. (See Lemma \ref{lemma:PC limited difference} in Appendix \ref{subsubsection: tools for the proof of theorem ub pc})
\begin{corollary}\label{corollary:upper bound 3}
    For any pessimistic algorithm and number of samples $N\geq \widetilde{O}(\frac{H}{P})$, we have
    \begin{equation*}
        \mathbf{V}_0^* - \mathbf{V}_0^{\up}\leq 2\sum_{h=1}^H\mathbb{E}_{\pi^*}\left[\max\{0,\mathbf{E}_h(s_h,a_h) -\frac{\Var_{\hat{P}_{h,s,a}}(\uv^{\up}_{h+1})}{H^2}\gap - \frac{\gap}{4H}\}\right].
    \end{equation*}
\end{corollary}

To proceed from Lemma \ref{theorem:clip theorem(informal)}, we will utilize the following inequality
$$\ddot{\mathbf{E}}_h(s,a)=\max\{0,\mathbf{E}_h(s,a)-\epsilon_h(s,a)\} \leq \frac{\mathbf{E}_h^2(s,a)}{\epsilon_h(s,a)}.$$
Then with different choices of $\epsilon_h(s,a)$ used in Corollary \ref{corollary:upper bound 1} and Corollary \ref{corollary:upper bound 3} and the bonus function used in the VI-LCB algorithm, we can derive Theorem \ref{theorem:ub c} and Theorem \ref{theorem:ub pc}. Here we briefly explain the $\widetilde{O}(1/\epsilon)$ dependence. We have $\mathbf{E}_h(s,a)\leq O(b_h(s,a))$, where $b_h(s,a)$ is the pessimistic bonus scaling as $O(1/\sqrt{n})$ (ignoring other dependences). Then by Lemma \ref{theorem:clip theorem(informal)} and the above inequality, we can achieve an $\widetilde{O}(1/\epsilon)$ bound immediately. A more detailed proof is provided in Appendix \ref{appendix:B}.

 Our technique is similar to the clip techique \citep{clip2019} in online MDP as both of them threshold the estimates (deficit here and surplus in their work) by some $\Theta(\gap)$ terms. \citet{clip2019} further clip another $\frac{\mathrm{gap}_h(s,\bar{a})}{H}$ term, and we can actually achieve that as well, but it does not improve the sample complexity in the offline setting as we do not need to explore all the actions. In addition, we develop a new thresholding function based on the empirical variance, which can improve the $H$ factor by utilizing the Bernstein's inequality and the total variance technique. We believe this new technique can also be applied to the online setting and improve the sample complexity there.


\section{Gap-dependent Lower Bounds}\label{sec:lower bounds}
\label{ssection:lower bounds}
In this section, we provide two lower bounds for uniform optimal coverage assumption (Assumption \ref{assumption:uniform optimal policy coverage}) and optimal policy coverage assumption (Assumption \ref{assumption:relative optimal policy coverage}) respectively. Our lower bounds show that Theorem \ref{theorem:ub p} and Theorem \ref{theorem:ub c} are optimal up to $H$ factors and logarithm terms. We begin with a general lower bound. \add{Here the \emph{offline learning algorithm}, $\mathbf{ALG}$, is defined as the algorithm that takes a dataset $\mathcal{D}$ as input and output a policy $\hat{\pi}$. Note that $\mathbf{ALG}$ can be stochastic.}
\begin{theorem}
 \label{theorem: main lb the act1}
 There exists some absolute constant $C$, such that for any $A\geq 3,S\geq 2,H\geq 2,\tau<\frac{1}{2},\lambda \leq\frac{1}{3}, \add{\lambda_1\geq 2}$ and offline RL algorithm $\mathbf{ALG}$, 
 if the number of sample trajectories satisfies
 \begin{equation*}
 N\leq C\cdot\frac{HS\modify{\lambda_1}}{\lambda\tau^2},
 \end{equation*}
 there exists a MDP instance $\mathcal{M}$ and a behavior policy $\mu$ with $\gap=\tau$, $P\geq \frac{\lambda}{eS\modify{\lambda_1}}$, $C^*\leq \modify{\lambda_1}$ such that 
 the output policy $\hat{\pi}$ of $\mathbf{ALG}$ has expected suboptimality
 \begin{equation*}
 \mathbb{E}_{\mathcal{M},\mu,\mathbf{ALG}}[\mathbf{V}^*_0-\mathbf{V}^{\hat{\pi}}_0]\geq \frac{\lambda H\tau}{12}.
 \end{equation*}
\end{theorem}
By choosing $\lambda=\modify{1/3}$, $\modify{\lambda_1 = \frac{1}{3ePS}}$ and $\tau=\gap$, Theorem \ref{theorem: main lb the act1} indicates the following corollary.
\begin{corollary} 
 For any given instance coefficients $(H,\gap,\modify{P\leq \frac{1}{6eS}})$, target suboptimality $\epsilon\lesssim H\gap$ and any offline reinforcement learning algorithm $\mathbf{ALG}$, 
 there exists an instance $(\mathcal{M},\mu)$ such that if the number of trajectory samples satisfies
 \begin{equation*}
 N\leq C\cdot\frac{H}{P\mathrm{gap}^2_{\min}},
 \end{equation*}
 the output policy $\hat{\pi}$ would have expected suboptimality more than $\epsilon$, i.e.,
 \begin{equation*}
 \mathbb{E}_{\mathcal{M},\mu, \mathbf{ALG}}[\mathbf{V}^*_0-\mathbf{V}^{\pi^*}_0]\geq  \frac{1}{36}H\gap \geq \epsilon
 \end{equation*}
\end{corollary}
 By choosing $\lambda=\frac{12\epsilon}{H\gap}$, $\modify{\lambda_1 =} C^*$, $\tau=\gap$, Theorem \ref{theorem: main lb the act1} indicates the following corollary.
\begin{corollary}
 For any given instance coefficients $(H,S,\gap, C^*\add{\geq 2})$, target suboptimality $\epsilon\lesssim H\gap$ and any offline learning algorithm $\mathbf{ALG}$, 
 there exists an instance $(\mathcal{M},\mu)$ such that if the number of  trajectory samples satisfies
 \begin{equation*}
 N\leq C\cdot\frac{H^2SC^*}{\epsilon\gap},
 \end{equation*}
 the output policy $\hat{\pi}$ would have expected suboptimality more than $\epsilon$, i.e.,
 \begin{equation*}
 \mathbb{E}_{\mathcal{M},\mu,\mathbf{ALG}}[\mathbf{V}^*_0-\mathbf{V}^{\pi^*}_0]\geq \epsilon.
 \end{equation*}
\end{corollary}
 The proof of the corollaries can be found in Appendix \ref{appendix:lb results}. We construct a family of MDPs that 
involve solving $HS$ independent bandit problems, where each bandit requires $\widetilde{\Omega}(\frac{HA}{\tau^2})$ visits so that the estimation error can be bounded by $\tau$. Similar constructions have been made in \citet{nearoptimalprovable2021(strongeuniformcoverage),dann2017lowerbound1,LCBVI2021}
. Our key observation is that the initial state distribution can linearly determine the uniform optimal policy coverage coefficient $P$ and the
final suboptimality. \add{Also, our proof can lead to another version of lower bound, where a $\Omega(\epsilon)$ suboptimality occurs with a constant probability. See Appendix \ref{sssection: prob lcb theo} for details.}

\section{Conclusion}
\label{section:discussions}
We presented a systematic study on gap-dependent upper and lower bounds for offline reinforcement learning. 
Depending on different assumptions, the rates can be improved from $\widetilde{O}(1/\epsilon^2)$ in the worst to $\widetilde{O}\left(1/\epsilon\right)$ or even independent of $1/\epsilon$.

One open question is that there still a gap between the upper and lower bounds in terms of $H$. We note that this gap also exists in the online setting~\citep{clip2019,xuhaike2021}.
Another direction to generalize our results to the function approximation setting~\citep{he2021logarithmic}.

\bibliographystyle{plainnat}
\bibliography{gap.bib}
\section{Upper Bound with Gap-dependent Analysis}
\label{appendix:A}
We begin with the proof of thresholding technique. In this section, definitions are restated for completeness.
\subsection{Definitions}
We first restate the notations. 
\begin{definition}[Pessimistic algorithms]
    \label{def: pessimistic alg}
    An offline learning algorithm with output policy $\up$ is pessimistic if with probability at least $1-\delta$, the following arguments hold,
    \begin{enumerate}
        \item It maintains a pessimistic estimate  $\uq$ of the true $Q^*$.
        \item $\uq$ is the optimal Q function of an imaginary MDP $\underline{\mathcal{M}}=(\mathcal{S},\mathcal{A},H,\mathcal{P},p_0,\underline{r})$, where $\underline{r}_h(s,a)\leq r_h(s,a)$ for all  $(h,s,a)\in[H]\times \mathcal{S}\times\mathcal{A}$.
        \item $\up$ is the greedy policy with respect to $\uq$. 
    \end{enumerate}
\end{definition}
\begin{definition}[Pessimistically estimated MDP]
    \label{definition: pessimistically estimated MDP}
    For a given successful pessimistic algorithm execution instance, where the arguments in Definition \ref{def: pessimistic alg} are simultaneously satisfied, we call $\underline{\mathcal{M}}=(\mathcal{S},\mathcal{A},H,\mathcal{P},p_0,\underline{r})$ the pessimistically estimated MDP. At the same time,
    $\uv$, $\uq$ are the corresponding value functions and Q functions. We use $\up$ to refer to the returned policy, which is optimal over $\underline{\mathcal{M}}$.
\end{definition}

And sometimes we use $\uq = \uq^{\up},\ \uv=\uv^{\up}$ without superscript indicating the policy. We'll show that this notation matches the definition in Algorithm \ref{alg:LCB} so there is no need worrying about any possible confusion.
\begin{definition}[Deficit]
    For a pessimistically estimated MDP $\underline{\mathcal{M}}=(\mathcal{S},\mathcal{A},H,\mathcal{P},p_0,\underline{r})$, we define deficit to be
    \begin{equation*}
        \mathbf{E}_h(s,a) \triangleq r_h(s,a) - \underline{r}_h(s,a). 
    \end{equation*}
\end{definition}
\begin{definition}[Not-so-pessimistic MDP]
    For a given set of $\epsilon_{h}(s,a)$, a pessimistically estimated MDP $\underline{\mathcal{M}}=(\mathcal{S},\mathcal{A},H,\mathcal{P},p_0,\underline{r})$, define
    \begin{equation*}
        \ddot{r}_h(s,a) \triangleq r_h(s,a) - \max\{0,\mathbf{E}_h(s,a) -\epsilon_{h}(s,a)\}.
    \end{equation*}
    Then we call $\ddot{\mathcal{M}}=(\mathcal{S},\mathcal{A},H,\mathcal{P},p_0,\ddot{r})$ not-so-pessimistic MDP. At the same time,
    $\ddot{\mathbf{V}}$ is the corresponding value functions. 
\end{definition}
\subsection{Main Theorem}
\label{ssection: main theorem}
For conciseness, we will use $a^*$ and $\underline{a}$ to stand for $\pi_h^*(s)$ and $\up_h(s)$ respectively when it introduces no confusion.
To formally present the deficit thresholding technique, we define a event $\xi_{gap}$.
\begin{definition}[Gap restriction event]
\label{def: gap restriction event}
    For a given set of $\epsilon_{h}(s,a)$, event $\xi_{gap}$ is defined to be the event such that for all optimal policy $\pi^*$, $h\in[H]$ and $s\in\S$, 
    \begin{equation*}
        \ddot{\mathbf{V}}_h^*(s)\leq \uv^*_h(s) + \frac{\gap}{2}. 
    \end{equation*}
\end{definition}
Note that $\xi_{gap}$ depends on the value of $\epsilon_{h}(s,a)$, and the definition of $\epsilon_{h}(s,a)$ may involve 
randomness. In the following proof of Corollary \ref{corollary:upper bound 2}, we will set $\epsilon_{h}(s,a)=\frac{\gap}{2H}$. 
\begin{theorem}[Deficit thresholding]
\label{theorem:main theorem}
When event $\xi_{gap}$ happens, there exists an optimal policy $\pi^*$, such that replacing $\mathbf{V}_0^{\up}$ with $\ddot{\mathbf{V}}_0^*$ only harms the difference up to a constant factor,
\begin{equation*}
    \mathbf{V}_0^*-\mathbf{V}_0^{\up}\leq 2(\mathbf{V}_0^*-\ddot{\mathbf{V}}^*_0).
\end{equation*}
\end{theorem}
Rigorous proof is deferred to Appendix \ref{subsection:proof of main theorem}
\begin{corollary}
    \label{corollary:upper bound 2}
    For a pessimistic algorithm running instance, there exists a deterministic optimal policy $\pi^*$, such that
    \begin{equation*}
        \mathbf{V}_0^* - \mathbf{V}_0^{\up}\leq 2\sum_{h=1}^H\mathbb{E}_{\pi^*}\left[\max\{0,\mathbf{E}_h(s_h,a_h)- \frac{\gap}{2H}\}\right]. 
    \end{equation*}
\end{corollary}
\begin{proof}
    With Theorem \ref{theorem:main theorem}, we just need to prove that $\epsilon_h(s,a)=\frac{\gap}{2H}$ indicates $\xi_{gap}$.\\
    Note that $\mathbf{E}_h(s,a^*)\leq \ddot{\mathbf{E}}_h(s,a^*) + \epsilon_{h}(s,a^*)=\ddot{\mathbf{E}}_h(s,a^*)+\frac{\gap}{2H}$. Therefore, for all optimal policy $\pi^*$, we have
    \begin{align*}
        \ddot{\mathbf{V}}^*_h(s)-\uv^*_h(s)&=\sum_{h'=h}^H\mathbb{E}_{\pi^*, s_h=s}[-\ddot{\mathbf{E}}_{h'}(s_{h'},a_{h'})+\mathbf{E}_{h'}(s_{h'},a_{h'})]\\
        &\leq \sum_{h'=h}^H\mathbb{E}_{\pi^*, s_h=s}[\frac{\gap}{2H}]\\
        &\leq \frac{\gap}{2}. 
    \end{align*}
\end{proof}

\subsection{Value/Q Function Ranking Lemma}
The following lemmas will be frequently used throughout the proof of Theorem \ref{theorem:main theorem} and upper bounds.
\begin{lemma}[Overall size relationships of value functions]
    \label{lemma:value function rank}
    When $\xi_{gap}$ happens,
    different value functions satisfy that for any optimal policy $\pi^*$, we have
    \begin{equation*}
        \mathbf{V}^*\geq \begin{cases}
            \mathbf{V}^{\up}\geq \uv^{\up} \\
            \ddot{\mathbf{V}}^*
        \end{cases}\geq \uv^*\geq \ddot{\mathbf{V}}^* - \frac{\gap}{2}. 
    \end{equation*}
    Here $\mathbf{V}\geq \mathbf{V}'$ means $\mathbf{V}_h(s)\geq \mathbf{V}'_h(s)$ for all $(h,s)\in[H]\times\mathcal{S}$. 
\end{lemma}
\begin{proof}
    We study each inequality one by one:\\
    $\mathbf{V}^*\geq \mathbf{V}^{\up}$: $\pi^*$ is a optimal policy over $\mathcal{M}$. \\
    $\mathbf{V}^{\up}\geq \uv^{\up}$: this follows from $\underline{r}\leq r$,
    \begin{equation*}
        \mathbf{V}^{\up}_h(s) = \sum_{h'=h}^H\mathbb{E}_{\up}[r_{h'}(s_{h'},a_{h'})\mid s_h =s]\geq \sum_{h'=h}^H\mathbb{E}_{\up}[\underline{r}_{h'}(s_{h'},a_{h'})\mid s_h =s] =\uv^{\up}_h(s). 
    \end{equation*}
    $\uv^{\up}\geq \uv^*$: $\up$ is a optimal policy over $\underline{\mathcal{M}}$. \\
    $\mathbf{V}^*\geq \ddot{\mathbf{V}}^*$: this follows from $\ddot{r}\leq r$,
    \begin{equation*}
        \mathbf{V}^{*}_h(s) = \sum_{h'=h}^H\mathbb{E}_{\pi^*}[r_{h'}(s_{h'},a_{h'})\mid s_h =s]\geq \sum_{h'=h}^H\mathbb{E}_{\pi^*}[\ddot{r}_{h'}(s_{h'},a_{h'})\mid s_h =s] =\ddot{\mathbf{V}}^{*}_h(s). 
    \end{equation*}
    $\ddot{\mathbf{V}}^*\geq \uv^*$: this follows from $\underline{r}\leq \ddot{r}$,
    \begin{equation*}
        \ddot{\mathbf{V}}^{*}_h(s) = \sum_{h'=h}^H\mathbb{E}_{\pi^*}[\ddot{r}_{h'}(s_{h'},a_{h'})\mid s_h =s]\geq \sum_{h'=h}^H\mathbb{E}_{\pi^*}[\underline{r}_{h'}(s_{h'},a_{h'})\mid s_h =s] =\uv^{*}_h(s). 
    \end{equation*}
    $\uv^*\geq \ddot{\mathbf{V}}^* - \frac{\gap}{2}$: this is just the definition of $\xi_{gap}$.
\end{proof}
\begin{lemma}[Overall size relationships of Q functions]
    \label{lemma:Q function rank}
    When $\xi_{gap}$ happens,
    different Q functions satisfy that for any optimal policy $\pi^*$, we have
    \begin{equation*}
        \mathbf{Q}_h^*(s,a^*)\geq \mathbf{Q}_h^*(s,\underline{a})\geq\mathbf{Q}_h^{\up}(s,\underline{a})\geq 
            \uq_h^{\up}(s,\underline{a}) 
        \geq \uq_h^{\up}(s,a^*)\geq \uq_h^*(s,a^*). 
    \end{equation*}
\end{lemma}
\begin{proof}
    We study each inequality one by one:\\
    $\mathbf{Q}_h^*(s,a^*)\geq \mathbf{Q}_h^*(s,\underline{a})$: $a^*$ is the optimal action at $(h,s)$ over $\mathcal{M}$. \\
    $\mathbf{Q}_h^*(s,\underline{a})\geq\mathbf{Q}_h^{\up}(s,\underline{a})$: this follows from $\mathbf{V}^*\geq \mathbf{V}^{\up}$ in Lemma \ref{lemma:value function rank},
    \begin{equation*}
        \mathbf{Q}_h^*(s,\underline{a}) = \mathbb{E}_{s'\sim P_{h,s,\underline{a}}}[\mathbf{V}^*_{h+1}(s')]\geq \mathbb{E}_{s'\sim P_{h,s,\underline{a}}}[\mathbf{V}^{\up}_{h+1}(s')]=\mathbf{Q}_h^{\up}(s,\underline{a}). 
    \end{equation*}
    $\mathbf{Q}_h^{\up}(s,\underline{a})\geq \uq_h^{\up}(s,\underline{a})$: this follows from $\mathbf{V}^{\up}\geq \uv^{\up}$ in Lemma \ref{lemma:value function rank}. \\
    $\uq_h^{\up}(s,\underline{a})\geq \uq_h^{\up}(s,a^*)$: $\underline{a}$ is the optimal action at $(h,s)$ over $\underline{\mathcal{M}}$. \\
    $\uq_h^{\up}(s,a^*)\geq \uq_h^*(s,a^*)$: this follows from $\uv^{\up}\geq \uv^{*}$ in Lemma \ref{lemma:value function rank},
    \begin{equation*}
        \uq_h^{\up}(s,a^*) = \mathbb{E}_{s'\sim P_{h,s,a^*}}[\uv^{\up}_{h+1}(s')]\geq \mathbb{E}_{s'\sim P_{h,s,a^*}}[\uv^{*}_{h+1}(s')]=\uq_h^{*}(s,a^*). 
    \end{equation*}
\end{proof}
\subsection{Proof of Theorem \ref{theorem:main theorem}}
\label{subsection:proof of main theorem}
\begin{proof}
In this proof, we choose the $\pi^*$ according to the given $\up$,
\begin{equation*}
    \pi^*_h(s) = \begin{cases}
        \up_h(s) & \up_h(s)\text{ is optimal,}\\
        \text{arbitrary optimal action} &\up_h(s)\text{ is not optimal.}
    \end{cases}
\end{equation*}
So that every time $\pi^*$ disagrees with $\up$, the choice made by $\up$ must be suboptimal.
The intuition is that we only consider the cases where $\pi^*$ and $\up$ have different opinions. 
To begin with, we define a set of prefix trajactories for any two given deterministic policies over MDPs that only differ in rewards:
\begin{align*}
    \Psi(\pi_1,\pi_2) = &\{(s_1,a_1,\cdots,s_k)\mid \pi_{1,i}(s_i)=\pi_{2,i}(s_i)=a_i,\forall i=1,2,\cdots k-1, \\
    &\pi_1(s_k)\neq \pi_2(s_k) \text{ or } k=H\}.
\end{align*}
And we use $P^{\pi}_{\psi}$ to denote the probability that we can get a prefix trajactory $\psi=(s_1,a_1,\cdots,s_k)$ with a deterministic policy $\pi$,
\begin{equation*}
    P^{\pi}_{\psi}\triangleq p_0(s_1)\prod_{h=1}^{k-1}p_h(s_h,\pi_h(s_h), s_{h+1}). 
\end{equation*}
Notice that for any given trajectory $\xi$ and policy $\pi_1, \pi_2$, there is exactly one prefix trajectory $\psi \in \Psi(\pi_1,\pi_2)$ being the prefix of 
$\xi$, which ends at the first time $\pi_1$ disagrees with $\pi_2$. Denote 
the length and the last state of a trajactory to be $2h_{\psi}-1$ and $s_{\psi}$, and set the 
cumulative reward over $\psi$ under a given deterministic reward function $r_h(s,a)$ to be $r_\psi$, we can write the value function in the form 
\begin{equation*}
    V_0^{\pi_1}=\sum_{\psi\in\Psi(\pi_1,\pi_2)}P_{\psi}^{\pi_1}\left(r_{\psi} + \mathbf{V}_{h_{\psi}}^{\pi_1}(s_{\psi})\right). 
\end{equation*}
Also, notice that because $\pi_1$ and $\pi_2$ agrees on all the decisions in $\psi\in\Psi(\pi_1,\pi_2)$, we always have $P_\psi^{\pi_1}=P_\psi^{\pi_2}$ for $\psi\in\Psi(\pi_1,\pi_2)$.
Now we have 
\begin{align*}
    &\mathbf{V}_0^*-\mathbf{V}_0^{\up}=\sum_{\psi \in \Psi(\pi^*,\up)}P_\psi^{\pi^*}(\mathbf{V}_{h_\psi}^*(s_\psi)-\mathbf{V}_{h_\psi}^{\up}(s_\psi))\\
    &=\sum_{\psi \in \Psi(\pi^*,\up)}P_\psi^{\pi^*}(\mathbf{Q}_{h_\psi}^*(s_\psi, \underline{a})+\mathrm{gap}_{h_\psi}(s_{\psi}, \underline{a})-\mathbf{V}_{h_\psi}^{\up}(s_\psi)). 
\end{align*}
Then we prove a statement that $\forall \psi\in\Psi(\pi^*,\up)$,
\begin{equation}
    \mathbf{Q}_{h_\psi}^*(s_\psi, \underline{a})+\mathrm{gap}_{h_\psi}(s_{\psi}, \underline{a})-\ddot{\mathbf{V}}_{h_\psi}^{*}(s_\psi)\geq\mathbf{Q}_{h_\psi}^*(s_\psi, \underline{a})+\frac{1}{2}\mathrm{gap}_{h_\psi}(s_{\psi}, \underline{a})-\mathbf{V}_{h_\psi}^{\up}(s_\psi). \label{state: lbo2}
\end{equation}
When $\underline{a} = a^*$, the only possibility is that $h_\psi=H$, then RHS$=$0. While the LHS is always non-negative because $\mathbf{Q}_{h_\psi}^*(s_\psi,\underline{a})=\mathbf{V}^{*}_{h_\psi}(s_\psi)\geq \ddot{\mathbf{V}}_{h_\psi}^*(s_\psi)$ (Lemma \ref{lemma:value function rank}).\\
When $\underline{a} \neq a^*$, event $\xi_{gap}$ guarantees that 
\begin{align*}
    \mathbf{Q}_{h_\psi}^*(s_\psi, \underline{a})+\mathrm{gap}_{h_\psi}(s_{\psi}, \underline{a})-\ddot{\mathbf{V}}_{h_\psi}^{*}(s_\psi)&\geq \mathbf{Q}_{h_\psi}^*(s_\psi, \underline{a})+\mathrm{gap}_{h_\psi}(s_{\psi}, \underline{a})-\uv_{h_\psi}^{*}(s_\psi)-\frac{\gap}{2}\\
    &\geq \mathbf{Q}_{h_\psi}^*(s_\psi, \underline{a})+\frac{1}{2}\mathrm{gap}_{h_\psi}(s_{\psi}, \underline{a})-\mathbf{V}_{h_\psi}^{\up}(s_\psi). 
\end{align*}
The inequality uses that $\gap\leq \mathrm{gap}_{h_\psi}(s_\psi,\underline{a})$ and that $\uv^*\leq \uv^{\up}
\leq \mathbf{V}^{\up}$.
At the same time, we can decompose $\mathbf{V}_0^*-\ddot{\mathbf{V}}_0^*$ in a similar way,
\begin{align}
    \mathbf{V}_0^*-\ddot{\mathbf{V}}_0^*&=\sum_{\psi\in \Psi(\pi^*, \up)}P_\psi^{\pi^*}(r_\psi-\ddot{r}_\psi+\mathbf{V}_{h_\psi}^*(s_\psi)-\ddot{\mathbf{V}}_{h_\psi}^*(s_\psi))\notag\\
    &\geq\sum_{\psi \in \Psi(\pi^*,\up)}P_\psi^{\pi^*}(\mathbf{Q}_{h_\psi}^*(s_\psi, \underline{a})+\mathrm{gap}_{h_\psi}(s_{\psi}, \underline{a})-\ddot{\mathbf{V}}_{h_\psi}^{*}(s_\psi))\label{state: lbo1}\\
    &\geq\sum_{\psi \in \Psi(\pi^*,\up)}P_\psi^{\pi^*}(\mathbf{Q}_{h_\psi}^*(s_\psi, \underline{a})+\frac{1}{2}\mathrm{gap}_{h_\psi}(s_{\psi}, \underline{a})-\mathbf{V}_{h_\psi}^{\up}(s_\psi))\label{state: lbo3}\\
    &\geq \frac{1}{2}\sum_{\psi \in \Psi(\pi^*,\up)}P_\psi^{\pi^*}(\mathbf{Q}_{h_\psi}^*(s_\psi, \underline{a})+\mathrm{gap}_{h_\psi}(s_{\psi}, \underline{a})-\mathbf{V}_{h_\psi}^{\up}(s_\psi))\label{state: lbo4}\\
    &=\frac{1}{2}(\mathbf{V}_0^*-\mathbf{V}_0^{\up}). \notag
\end{align}
\eqref{state: lbo1} results from the fact that $r$ is always larger than or equal to $\ddot{r}$. \eqref{state: lbo3} just makes use of \eqref{state: lbo2}.
\eqref{state: lbo4} uses $\mathbf{Q}^*_h(s,\underline{a})\geq \mathbf{Q}^{\up}_h(s,\underline{a})= \mathbf{V}^{\up}_h(s)$(Lemma \ref{lemma:Q function rank}).
\end{proof}
\section{VI-LCB based analysis}
\label{appendix:B}
\subsection{Algorithm Sketch and Notations}
Algorithm used here is Lower Confidence Bound Value Iteration(VI-LCB)\citep{LCBVI2021} with subsampling trick and Berstein-style bonus. The 
basic idea of LCB is to pessimistically estimate the Q function so that the algorithm won't over estimate some hardly seen suboptimal actions in dataset. The subsampling 
trick introduced by \citet{settling2022} helps solve the independence problem between $\hat{P}_{h}$ and $\uv_{h+1}^{\up}$, which 
avoid separating the dataset into H parts, resulting in one H dependency removed in final complexity.

Here we understand dataset as a set of transitions in the form $(h,s,a,s')$ that allows duplicates. When we say that the dataset contains a trajactory $(s_1,a_1,\cdots, s_h,a_h)$, it means that the dataset contains all the decomposed transitions $\{(h,s_h,a_h,s_{h+1})\}_{h=1,\cdots H}$. 
Also note that $\mathcal{M}$ has deterministic rewards in our setting, so the reward function can be easily derived as long as the $(h,s,a)$ tuple is visited for at least once. And if $(h,s,a)$ is not contained in $\mathcal{D}$, the 
algorithm output wouldn't be influenced by the value of $r_h(s,a)$, and we can set $r_h(s,a)=0$. So we assume that the reward function is known from the beginning.
\IncMargin{1em}
\begin{algorithm}\label{alg:apx lcb}
    \DontPrintSemicolon
    \SetKwData{Left}{left}\SetKwData{This}{this}\SetKwData{Up}{up}
    \SetKwFunction{Union}{Union}\SetKwFunction{FindCompress}{FindCompress}
    \SetKwInOut{Input}{input}\SetKwInOut{Output}{output}

    \Input{Dataset $\mathcal{D}_0$, reward function $r$}
    \BlankLine
    set $\uq_{H+1}^{\up}(s,a)=0$\;
    set $\uv_{H+1}^{\up}(s,a)=0$\;
    \For{$h\leftarrow H$ \KwTo $1$}{
        compute the empirical transition kernel $\hat{P}_h$\;
        $\hat{P}_{h,s,a}(s') = \frac{N_h(s,a,s')}{N_h(s,a)}$ with $0/0=0$\;
        \For{$s\in\mathcal{S}, a\in\mathcal{A}$}{
            $b_h(s,a)\leftarrow C_b\sqrt{\frac{\Var_{\hat{P}_{h,s,a}}(\uv^{\up}_{h+1})\iota}{N_h'(s,a)}} + C_b\frac{H\iota}{N_h'(s,a)}$, where $N_h'(s,a) = N_h(s,a)\vee \iota $\;
            $\uq_h^{\up}(s,a) \leftarrow \max\{0,r_h(s,a) + \hat{P}_{h,s,a}^\top\uv^{\up}_{h+1} - b_h(s,a)\}$
        }
        \For{$s\in\mathcal{S}$}{
            $\uv_h^{\up}(s)\leftarrow\underset{a\in\mathcal{A}}{\max}\ \uq^{\up}_{h}(s,a)$\;
            $\up_h(s)\leftarrow \underset{a\in\mathcal{A}}{\arg\max}\ \uq^{\up}_h(s,a)$ 
        }
    }
    \Output{policy $\up$}
\caption{VI-LCB}
\end{algorithm}
\DecMargin{1em}
\IncMargin{1em}
\begin{algorithm}\label{alg:apx lcb sub}
    \DontPrintSemicolon
    \SetKwData{Left}{left}\SetKwData{This}{this}\SetKwData{Up}{up}
    \SetKwFunction{Union}{Union}\SetKwFunction{FindCompress}{FindCompress}
    \SetKwInOut{Input}{input}\SetKwInOut{Output}{output}

    \Input{Dataset $\mathcal{D}$, reward function $r$}
    \BlankLine
    Split $\mathcal{D}$ into 2 halves containing same number of sample trajectories, $\mathcal{D}^{\mathrm{main}}$ and $\mathcal{D}^{\mathrm{aux}}$\;
    $\mathcal{D}_0=\{\}$\;
    \For{$(h,s)\in[H]\times \mathcal{S}$}{
        $N^{\mathrm{trim}}_h(s)\leftarrow \max\{0, N^{\mathrm{aux}}_h(s) - 10\sqrt{N_h^{\mathrm{aux}}(s)\log\frac{HS}{\delta}}\}$\;
        Randomly subsample $\min\{N^{\mathrm{trim}}_h(s), N^{\mathrm{main}}_h(s)\}$ samples of transition from $(h,s)$ from $\mathcal{D}^{\mathrm{main}}$ to add 
        to $\mathcal{D}_0$
    }
    $\up\leftarrow \text{VI-LCB}(\mathcal{D}_0, r)$\;
    \Output{policy $\up$}
\caption{Subsampled VI-LCB}
\end{algorithm}
\DecMargin{1em}
\paragraph*{Notations in VI-LCB.}
In the algorithm, $N_h(s,a)$ refers to the number of sample transitions starting from state $s$, taking action $a$ at time step $h$ of some given dataset, and $N_h(s)=\sum_{a\in\mathcal{A}}N_h(s,a)$. Superscripts stand for the dataset. See \citet{settling2022} for a more detailed description of the algorithm. 

The proof of independence between samples with different $h$ in $\mathcal{D}_0$ is omitted here, and we will not need it directly because the proof of Lemma \ref{lemma:settlinglemma6} suggests it.

Note that different from the notation in Algorithm \ref{alg:LCB}, we use $\uv^{\up}$ and $\uq^{\up}$ instead of $\uv$ and $\uq$ in Algorithm \ref{alg:apx lcb sub}. We do this to emphasize that $\uv$ and $\uq$ in Algorithm \ref{alg:LCB} directly satisfies the definitions of $\uv^{\up}$ and $\uq^{\up}$ in thresholding technique (Definition \ref{definition: pessimistically estimated MDP}), which will be rigorously proved in Lemma \ref{lemma:concentration event}.
To avoid unnecessary confusion or reading difficulty, $\uv$ and $\uq$ without superscript stands for the true optimal Q/value functions of $\underline{\mathcal{M}}$, i.e., $\uv^{\up}$ and $\uq^{\up}$, in the following proof.
\subsection{Proof Preparation}
To warm up, we first prove that VI-LCB perfectly matchs our definition of LCB-style algorithm. 
From the original paper of VI-LCB \citep{settling2022}, we quote a slightly modified version of their lemma 6, where constants and notations are changed, and $V$ is replaced with $\uv^{\up}_{h+1}$. 
\begin{lemma}[Transition estimation bound]
    \label{lemma:settlinglemma6}
    For any $1\leq h\leq H$, with probability at least $1-\frac{\delta}{2H}$, we have
    \begin{equation*}
        |(\hat{P}_{h,s,a}-P_{h,s,a})^\top \uv_{h+1} |\leq b_n(s,a) = C_b\sqrt{\frac{\Var_{\hat{P}_{h,s,a}}(\uv^{\up}_{h+1})\iota}{N_h(s,a)}} + C_b\frac{H\iota}{N_h(s,a)}, \qquad \forall (s,a)\in \mathcal{S}\times\mathcal{A}. 
    \end{equation*}
\end{lemma}
Proof for Lemma \ref{lemma:settlinglemma6} is omitted here. With the union bound, we have the inequality in Lemma \ref{lemma:settlinglemma6} holds for all $h\in[H]$ with probability over $1-\frac{\delta}{2}$. Also, Lemma 1 from original paper helps with the guarantee of the sample number.
\begin{lemma}
    \label{lemma:settling lemma 1}
    With probability over $1-\frac{\delta}{2}$, we have
    \begin{equation*}
        N_h(s,a)\geq C_{data}(Nd_h^{\mu}(s,a) - \sqrt{Nd_h^{\mu}(s,a)\iota} ),\qquad \forall (h,s,a)\in[H]\times\mathcal{S}\times\mathcal{A},
    \end{equation*}
    for some positive constant $C_{data}$. 
\end{lemma}
Proof of this lemma is also omitted, which is a direct result of Binomial concentration.
Then we can prove the concentration lemma, which serves as the basis of following analysis.
\begin{lemma}
    \label{lemma:concentration event}
If we run VI-LCB on a offline learning instance $(\mathcal{M},\mu)$,
with high probability (over $1-\delta$), the following event $\xi_{\mathrm{conc}}$ happens for some positive constant $C_d$:
\begin{enumerate}
    \item the execution instance satisfies the three arguments in Definition \ref{def: pessimistic alg}, and $0\leq\mathbf{E}_h(s,a)\leq 2b_h(s,a)$ for all $(h,s,a)\in[H]\times\mathcal{S}\times\mathcal{A}$. 
    \item $N_h'(s,a) = N_h(s,a)\vee \iota \geq C_{d}Nd_h^{\mu}(s,a)$ for all $ (h,s,a)\in[H]\times\mathcal{S}\times\mathcal{A}$. 
\end{enumerate}
\end{lemma}
\begin{proof}
    We just need to prove that both statements are true with probability over $1-\frac{\delta}{2}$ respectively, and then we can finish the proof by applying union bound.
    
    \textbf{Proof of statement 1: }
    $\uq$ in the definition of pessimistic algorithm matches the $\uq$ in VI-LCB. 
    We first prove that $\underline{\mathcal{M}}$ exists.
    With $\uq$ given, we can actually get a closed form of $\underline{r}$,
    \begin{align*}
        \underline{r}_h(s,a) &= \uq_h(s,a) - \mathbb{E}_{\up\mid s_h=h, a_h=a}[\uq_h(s_{h+1},a_{h+1})]\\
        &=\uq_h(s,a) - \sum_{s'\in\mathcal{S}}P_{h,s,a}(s')\uq_{h+1}(s',\underline{a})\\
        &=\uq_h(s,a) - P_{h,s,a}^{\top}\uv_{h+1}. 
    \end{align*}
    
    Then we can find that the definition of $\uv$ in the algorithm agrees with the one in Definition \ref{definition: pessimistically estimated MDP}, and we won't 
    distinguish between these two definitions in following induction.
    It remains to show that 
    \begin{equation}
        0\leq \mathbf{E}_h(s,a)\triangleq r_h(s,a) - \underline{r}_h(s,a)\leq 2b_h(s,a).  \label{state:lcbvi preparation}
    \end{equation}
    Both inequalities follow from Lemma \ref{lemma:settlinglemma6}. Recall that $\uq_h(s,a)=\max\{r_h(s,a) + \hat{P}_{h,s,a}^\top\uv_{h+1} - b_h(s,a) ,0\}$,
    \begin{align*}
        \uq_h(s,a)&=\max\{0, r_h(s,a) + \hat{P}_{h,s,a}^\top\uv_{h+1} - b_h(s,a)\}\\
        &= \max\{0,P_{h,s,a}^\top\uv_{h+1} + r_h(s,a) +(\hat{P}_{h,s,a}-P_{h,s,a})^\top \uv_{h+1} -b_h(s,a)\}\\
        &\leq \max\{0,P_{h,s,a}^\top\uv_{h+1} + r_h(s,a) \} = P_{h,s,a}^\top\uv_{h+1} + r_h(s,a). 
    \end{align*}
    Simple transformation of above inequality leads to
    \begin{equation*}
        r_h(s,a) - \underline{r}_h(s,a)\geq\uq_h(s,a) - P_{h,s,a}^\top\uv_{h+1} - \underline{r}_h(s,a) =0. 
    \end{equation*}
    The second inequality is also straight forward. We first unfold the definitions of $r$ and $\underline{r}$, then apply Lemma \ref{lemma:settlinglemma6} to get
    \begin{align*}
        r_h(s,a) - \underline{r}_h(s,a)&\leq \uq_h(s,a) -\hat{P}_{h,s,a}^\top\uv_{h+1}+ b_h(s,a) - (\uq_h(s,a) - P_{h,s,a}^\top\uv_{h+1})\\
        &= b_h(s,a) - (\hat{P}_{h,s,a}-P_{h,s,a})^\top \uv_{h+1} \\
        &\leq 2b_h(s,a). \\
    \end{align*}
    \textbf{Proof of statement 2:}
    We prove over the assumption of event: $N_h(s,a)\geq C_{data}(Nd_h^{\mu}(s,a) - \sqrt{Nd_h^{\mu}(s,a)\iota} )$ for all $ (h,s,a)\in[H]\times\mathcal{S}\times\mathcal{A}$, which is proved by Lemma \ref{lemma:settling lemma 1} to happen with probability over $1-\frac{\delta}{2}$. \\
    When $Nd_h^{\mu}(s,a) - \sqrt{Nd_h^{\mu}(s,a)\iota} \leq \frac{\iota}{C_{data}}$, simple calculation leads to 
    \begin{equation*}
        \sqrt{Nd_h^{\mu}(s,a)} \leq \frac{1 + \sqrt{1 + \frac{4}{C_{data}}}}{2}\sqrt{\iota}=\lambda\sqrt{\iota},
    \end{equation*}
    where $\lambda=\frac{1 + \sqrt{1 + \frac{4}{C_{data}}}}{2}$ is a constant larger than 1. Therefore 
    \begin{equation}
    \label{state:lcbprecase1}
        N_h(s,a) \vee \iota \geq \iota \geq \frac{1}{\lambda^2}Nd_h^\mu(s,a). 
    \end{equation}
    When $Nd_h^{\mu}(s,a) - \sqrt{Nd_h^{\mu}(s,a)\iota} \geq \frac{\iota}{C_{data}}$, simple calculation leads to
    \begin{align}
        \sqrt{Nd_h^{\mu}(s,a)} &\geq \frac{1 + \sqrt{1 + \frac{4}{C_{data}}}}{2}\sqrt{\iota}= \lambda\sqrt{\iota} \notag\\
        \Leftrightarrow  Nd_h^{\mu} (s,a) &- \sqrt{Nd_h^{\mu}(s,a)\iota} \geq (1-\frac{1}{\lambda})Nd_h^{\mu}(s,a) \notag \\
        \Rightarrow N_h(s,a)\vee \iota&\geq C_{data}(Nd_h^{\mu}(s,a)-\sqrt{Nd_h^{\mu}(s,a)\iota})\geq C_{data}(1-\frac{1}{\lambda})Nd_h^{\mu}(s,a).  \label{state:lcbprecase2}
    \end{align}
    Then together with \eqref{state:lcbprecase1} and \eqref{state:lcbprecase2}, and letting $C_d = C_{data}(1-\frac{1}{\lambda})\wedge 1$, we finish the proof of statement 2.
\end{proof}
\begin{lemma}
    \label{lemma: upper bound on clipped term}
        When $0\leq x\leq y$, for $\epsilon>0$
    \begin{equation*}
        \max\{0,x- \epsilon\} \leq \frac{y^2}{\epsilon}. 
    \end{equation*}
\end{lemma}
\begin{proof}
    When $x\leq \epsilon$, $\max\{0, x- \epsilon\}=0\leq \frac{y^2}{\epsilon}$.\\
    When $x > \epsilon$,
    \begin{equation*}
        \max\{0,x -\epsilon\}\leq y\leq y\cdot \frac{x}{\epsilon} \leq \frac{y^2}{\epsilon}. 
    \end{equation*}
\end{proof}
\subsection{Proof of Upper Bound with Relative Optimal Policy Coverage(Proof of Theorem \ref{theorem:ub c})}
This analysis is actually made with Hoeffding bonus for simplicity. Because Berstein bonus is larger than Hoeffding bonus up to log term,
\begin{equation}
    b_h(s,a) = C_b\sqrt{\frac{\Var_{\hat{P}_{h,s,a}}(\uv_{h+1}^{\up})\iota}{N_h'(s,a)}} + \frac{C_bH\iota}{N_h'(s,a)}\leq 2C_b\sqrt{\frac{H^2\iota^2}{N_h'(s,a)}}.
\end{equation}
With Corollary \ref{corollary:upper bound 2} and 
Lemma \ref{lemma: upper bound on clipped term},
\begin{align*}
    \mathbf{V}^*_0-\mathbf{V}_0^{\up}&\leq 2\sum_{h=1}^H\mathbb{E}_{\pi^*}[\max\{0,\mathbf{E}_h(s,a) - \frac{\gap}{2H}\}]  \tag{Corollary \ref{corollary:upper bound 2}}\\
    &=2 \sum_{h,s}d^*_h(s)\max\{0,\mathbf{E}_h(s,a)-\frac{\gap}{2H}\}\\
    &\lesssim \sum_{h,s} d_h^*(s)\frac{b_h^2(s,a^*)}{\frac{\gap}{2H}} \tag{Lemma \ref{lemma: upper bound on clipped term} and $\xi_{\mathrm{conc}}$}\\
    &\lesssim \sum_{h,s} d_h^*(s)\frac{H^3\iota^2}{N_h(s,a^*)'\gap} \\
    &\lesssim \sum_{h,s} d_h^*(s)\frac{H^3\iota^2}{Nd_h^{\mu}(s,a)\gap} \tag{$\xi_{\mathrm{conc}}$}\\
    &\lesssim \frac{1}{N}\sum_{h,s} d_h^*(s)\frac{H^3C^*\iota^2}{d^*_h(s)\gap} \tag{relative optimal coverage assumption}\\
    &=\frac{1}{N}\frac{H^4SC^*\iota^2}{\gap}.
\end{align*}
Therefore, under relative optimal policy coverage, the sample complexity bound can be 
\begin{equation*}
    N = O(\frac{H^4SC^*\iota^2}{\epsilon\ \gap}).
\end{equation*}
A similar proof in Section \ref{ssection:case with both assumptions} can be applied to prove this result by replacing all the $NP\gtrsim H\iota$ requirements with $N_h(s,a)'\triangleq N_h(s,a)\vee \iota\geq \iota$. So the strict bound without extra $\iota$ can be derived.
 To avoid redundancy, we omit the proof.
\begin{equation*}
    N =O(\frac{H^4SC^*\iota}{\epsilon\ \gap}).
\end{equation*}
\subsection{Proofs of Upper Bound with Uniform Optimal Policy Coverage(Proof of Theorem \ref{theorem:ub p})}
\label{subsection: proof of ub p}
Proof of Theorem \ref{theorem:ub p} does not necessarily involve the deficit thresholding technique introduced above. We just need to confirm that $\uq^{\up}_h(s,a^*)\geq \mathbf{V}^*_h(s) - 
\gap\geq \mathbf{V}^*_h(s,a')\geq \uq^{\up}_h(s,a')$, where $a'$ is any suboptimal action, to get a optimal policy. We first present the proof applying this idea, and then present a simpler proof by applying the thresholding technique.\\
\subsubsection{Proof without Deficit Thresholding Technique}
First we introduce a new definition,
\begin{align*}
    &d^{*}_{h\sim (h',s')}(s)\triangleq \mathbb{E}_{\pi^*}[\mathbb{I}\{s_h=s\}\mid s_{h'}=s'], \\
    &d^{*}_{h\sim (h',s')} \triangleq (d^{*}_{h\sim (h',s')}(s_1),\cdots, d^{*}_{h\sim (h',s')}(s_S))^\top &\text{for some certain order of states $s_1,s_2,\cdots s_S$}.
\end{align*}
And when there is no confusion, we use $d^{*'}_h$ to denote $d^{*}_{h\sim (h',s')}$.
\begin{lemma}[Part Decomposition]
\label{lemma:part decomposition}
    $\forall (h',s')\in [H]\times \mathcal{S}$, if the event $\xi_{\mathrm{conc}}$ happens, and $P>0$, then $\forall$ optimal policy $\pi^*$,
    \begin{equation*}
    \sum_{h=h'}^H\sum_sd^{*}_{h\sim (h',s')}(s)b_{h'}(s,a^*)\leq C_e\sqrt{\frac{H^3\iota}{NP}}+C_e\frac{H^2\iota}{NP},
    \end{equation*}
    where $C_e = \max\{4\frac{C_b}{\sqrt{C_d}},\frac{16C_b^2+12C_b}{C_d}, 1\}$.
\end{lemma}
With this lemma, we can further limit $\uq^{\up}_h(s,a^*)$. 
\begin{align*}
    \mathbf{V}_{h'}^*(s') - \uq_{h'}^{\up}(s',a^*)&\leq \mathbf{V}_{h'}^*(s')-\uq_{h'}^*(s',a^*) \tag{$\up$ is the optimal policy over $\underline{\mathcal{M}}$}\\
    & = \sum_{h=h'}^H\sum_sd_{h\sim (h',s')}^*(s)\mathbf{E}_h(s,a^*)\\
    &\leq 2\sum_{h,s}d_h^{*'}b_h(s,a^*)\tag{$\xi_{\mathrm{conc}}$}\\
    &\leq 2C_e\sqrt{\frac{H^3\iota}{NP}} + 2C_e\frac{H^2\iota}{NP}. \tag{Lemma \ref{lemma:part decomposition}}
\end{align*}
When $N\geq \frac{4C_e^2H^3\iota}{\lambda^2P}$ for some $\lambda\leq H$,
\begin{align*}
    &\mathbf{V}_{h'}^*(s') - \uq_{h'}^{\up}(s',a^*)\\
    &\leq \frac{\lambda}{2} + \frac{\lambda}{2}\frac{\lambda}{2C_eH}\\
    &\leq \lambda.
\end{align*}
Setting $\lambda=\gap$, we get the conclusion that 
\begin{equation*}
    N=\frac{4C_e^2H^3\iota}{\gap^2P} 
\end{equation*}
can make sure that the returned policy is one of the optimal policies with probability over $1-\delta$.

\subsubsection{Proof with Deficit Thresholding Technique}
By applying the Lemma \ref{lemma:PC limited difference} which is orginally developed for the proof of Theorem \ref{theorem:ub pc}, we shall directly prove that the suboptimality would be zero if $N>\frac{CH^3\iota}{P\gap}$.
Lemma \ref{lemma:PC limited difference} allows us to set $\epsilon_h(s,a)=C_{pac}\left(\frac{\Var_{\hat{P}_{h,s,a}(\uv^{\up}_{h+1})}}{H^2}+ \frac{1}{H}\right)\gap$. When $N\geq C\frac{H^3\iota}{P\gap^2}$, for any optimal policy $\pi^*$,
\begin{align*}
    b_h(s,a^*) &=C_b\sqrt{\frac{\Var_{\hat{P}_{h,s,a^*}}(\uv^{\up}_{h+1})\iota}{N_h'(s,a^*)}} + C_b\frac{H\iota}{N_h(s,a^*)}\\
    &\lesssim \sqrt{\frac{\Var_{\hat{P}_{h,s,a^*}}(\uv^{\up}_{h+1})\iota}{NP}} + \frac{H\iota}{NP}\tag{$\xi_{\mathrm{conc}}$}\\
    &\leq \sqrt{\frac{\Var_{\hat{P}_{h,s,a^*}}(\uv^{\up}_{h+1})\gap^2}{CH^3}} + \frac{\gap^2}{CH^2} \tag{$N\geq C\frac{H^3\iota}{P\gap^2}$}\\
    &\leq \frac{\Var_{\hat{P}_{h,s,a^*}}(\uv^{\up}_{h+1})\gap}{2CH^2} + \frac{\gap}{2H} + \frac{\gap}{CH}& a+b\geq 2\sqrt{ab},\ \gap\leq H\\
    &\lesssim \epsilon_h(s,a^*).
\end{align*}
Therefore, with a large enough global constant $C$, we have $\mathbf{E}_h(s,a^*)\leq 2b_h(s,a^*)\leq \epsilon_h(s,a^*)$ holds for any time-state pair. Together with Theorem \ref{theorem:main theorem},
\begin{align*}
    \mathbf{V}^*-\mathbf{V}^{\up}&\leq 2\mathbb{E}_{\pi^*}\left[\ddot{\mathbf{E}}_h(s,a)\right]=0.
\end{align*}

\subsubsection{Tools for the Proof of Lemma \ref{lemma:part decomposition}}
We first introduce a modified version of Lemma from \citet{settling2022}. Note that the proof of this lemma didn't involve any assumption about 
the data coverage, and is a pure mathmatical analysis. So the original proof is valid, and to avoid redundancy, we won't prove this lemma again.
\begin{lemma}
\label{lemma:tool for part decomposition}
    $\forall h\in[H]$, and any vector $V\in\mathbb{R}^S$ independent of $\hat{P}_h$ obeying $\|V\|_\infty\leq H$. With probability 
    at least $1-\delta$, one has 
    \begin{equation*}
        \Var_{\hat{P}_{h,s,a}}(V)\leq 2\Var_{P_{h,s,a}}(V) + \frac{5H^2\iota}{3N_h'(s,a)}
    \end{equation*}
    simultaneously for all $(s,a)\in \mathcal{S}\times \mathcal{A}$ obeying $N_h(s,a)> 0$
\end{lemma}
Modification lies in that we use $N'_h(s,a) $ to replace $N_h(s,a)$ in original version, for when $N_h(s,a)\leq \iota$, the inequalities hold trivially.  
Also we introduce the lemma needed to limit the overall variance. This lemma differs from \citet{settling2022}'s work from the definition of $d^{*'}_h$ to support our theorem.
\begin{lemma}[weighted variance sum]
    \label{lemma:variance sum}
    $\forall (h',s')\in [H]\times \mathcal{S}$, if the event $\xi_{\mathrm{conc}}$ happens, we have 
    \begin{equation*}
        \sum_{h=h'}^H\sum_sd^{*}_{h\sim (h',s')}(s)\Var_{P_{h,s,a}}(\uv_{h+1}^{\up}) \leq 4H\sum_{h=h'}^H\sum_sd^{*}_{h\sim (h',s')}(s)b_{h}(s,a^*) +2H^2.
    \end{equation*}
\end{lemma}
\begin{proof}
    Here we use $P_h^*\in \mathbb{R}^{S\times S}$ to denote the transition kernel of optimal policy, where $P_{h,(m,n)}^*$ is 
    the probability of transfer from $s_m$ to $s_n$ at step $h$ while applying the optimal policy $\pi^*$. $A \circ B$ refers to the Hadamard 
    product of $A$ and $B$.
    \begin{align}
        &\sum_{h=h'}^H\sum_sd^{*'}_{h}(s)\Var_{P_{h,s,a^*}}(\uv_{h+1}^{\up})=\sum_{h=h'}^Hd^{*'\top}_h (P_h^*\uv_{h+1}^{\up}\circ\uv_{h+1}^{\up}-(P_h^*\uv_{h+1}^{\up})\circ (P_h^*\uv_{h+1}^{\up}))\notag\\
        &=\sum_{h=h'}^Hd^{*'\top}_h(P_h^*\uv_{h+1}^{\up}\circ\uv_{h+1}^{\up}-\uv_h^{\up}\circ\uv_h^{\up}+\uv_h^{\up}\circ\uv_h^{\up} - (P_h^*\uv_{h+1}^{\up})\circ (P_h^*\uv_{h+1}^{\up}))\notag\\
        &=\sum_{h=h'}^{H}\left(d^{*'\top}_{h+1}\uv_{h+1}^{\up}\circ\uv_{h+1}^{\up} -d^{*'\top}_{h}\uv_h^{\up}\circ\uv_h^{\up}\right)+\sum_{h=h'}^Hd^{*'\top}_h(\uv_h^{\up}\circ\uv_h^{\up} - (P_h^*\uv_{h+1}^{\up})\circ (P_h^*\uv_{h+1}^{\up}))\notag\\
        &=0-d^{*'\top}_{h'}\uv_{h'}^{\up}\circ\uv_{h'}^{\up}+\sum_{h=h'}^Hd_h^{*'\top} (\uv_h^{\up}-P_h^*\uv_{h+1}^{\up})\circ(\uv_h^{\up}+P_h^*\uv_{h+1}^{\up})\notag\\
        &\leq \sum_{h=h'}^Hd_h^{*'\top} (\uv_h^{\up}-P_h^*\uv_{h+1}^{\up})\circ(\uv_h^{\up}+P_h^*\uv_{h+1}^{\up}). \label{state:part1}
    \end{align} 
    The above induction mainly uses the equality that $d_h^{*'}P^*_h=d_{h+1}^{*'}$ and non-negativity of $d_h^{*'}$.
    Because the concentration events $\xi_{\mathrm{conc}}$ guarantees that $b_h(s,a)\geq |(\hat{p}_h(s,a)-p_h(s,a))^\top\uv_{h+1}|$, 
    \begin{align*}
        &\uv_h^{\up}(s)-p_h(s,a^*)^\top\uv_{h+1}^{\up}\\
        &=\uv_h^{\up}(s)-\uq_h^{\up}(s,a^*) + r_h(s,a^*) - b_h(s,a^*)+(\hat{p}_h(s,a^*)-p_{h}(s,a^*))^\top\uv_{h+1}^{\up}\\
        &\geq 0 + 0 - b_h(s,a^*) - b_h(s,a^*) = -2b_h(s,a^*).
    \end{align*}
    Then we can continue from \eqref{state:part1} to get
    \begin{align*}
        &\sum_{h=h'}^H\sum_sd^{*'}_{h}(s)\Var_{P_{h,s,a}}(\uv_{h+1}^{\up})\leq \sum_{h,s}d_h^{*'\top} (\uv_h^{\up}-P_h^*\uv_{h+1}^{\up})\circ(\uv_h^{\up}+P_h^*\uv_{h+1}^{\up})\\
        &\leq \sum_{h=h'}^Hd_h^{*'\top} (\uv_h^{\up}-P_h^*\uv_{h+1}^{\up}+2 b_h(s,a^*)\mathbf{1})\circ(\uv_h^{\up}+P_h^*\uv_{h+1}^{\up})\\
        &\leq 2H\sum_{h=h'}^Hd_h^{*'\top} (\uv_h^{\up}-P_h^*\uv_{h+1}^{\up}+2 b_h(s,a^*)\mathbf{1})\\
        &=2H(d_{h'}^{*'\top}\uv_{h'}^{\up}-d_{H+1}^{*'\top}\uv_{H+1}^{\up}) + 4H\sum_{h'=h}^H\sum_sd^{*}_{h\sim (h',s')}(s)b_{h}(s,a^*)\\
        &\leq 2H^2 +4H\sum_{h'=h}^H\sum_sd^{*}_{h\sim (h',s')}(s)b_{h}(s,a^*).
    \end{align*}
\end{proof}
\subsubsection{Proof of Lemma \ref{lemma:part decomposition}}
\begin{proof}
    This proof is similar to the one in \citet{settling2022}. The difference lies in that we generalize the 
    conclusion to any part decomposition, while the original version only cares about the optimal policy distribution.\\
    First, it follows from Lemma \ref{lemma:tool for part decomposition} and inequality $\sqrt{a+b}\leq \sqrt{a}+\sqrt{b}$ that 
    \begin{align}
        \frac{1}{C_b}b_h(s,a)&=\sqrt{\frac{\Var_{\hat{P}_{h,s,a}}(\uv_{h+1}^{\up})\iota}{N_h'(s,a)}} + \frac{H\iota}{N_h'(s,a)}\notag\\
        &\leq \sqrt{\frac{2\Var_{P_{h,s,a}}(\uv_{h+1}^{\up})\iota + \frac{5H^2\iota}{3N_h'(s,a)}\iota}{N_h'(s,a)}} + \frac{H\iota}{N_h'(s,a)}\notag\\
        &\leq \sqrt{\frac{2\Var_{P_{h,s,a}}(\uv_{h+1}^{\up})\iota}{N_h'(s,a)}} +(1+\sqrt{\frac{5}{3}}) \frac{H\iota}{N_h'(s,a)}\notag\\
        &\leq 2\sqrt{\frac{\Var_{P_{h,s,a}}(\uv_{h+1}^{\up})\iota}{N_h'(s,a)}} + \frac{3H\iota}{N_h'(s,a)}. \label{state: bh ele treat}
    \end{align}
    Note that the concentration event $\xi_{\mathrm{conc}}$ guarantees that $N_h(s,a^*)\geq C_d Nd^{\mu}_h(s,a^*)\geq C_d NP$.
    Then we can use Cauchy-Schwarz Inequality to limit the variance term,
    \begin{align*}
        &\sum_{h,s}d_h^{*'}(s)\sqrt{\frac{\Var_{P_{h,s,a^*}}(\uv_{h+1}^{\up})\iota}{N_h'(s,a^*)}}\\
        &\leq \sqrt{\frac{\iota}{C_d NP}}\sum_{h,s}d_h^{*'}(s)\sqrt{\Var_{P_{h,s,a^*}}(\uv_{h+1}^{\up})}\tag{$\xi_{\mathrm{conc}}$}\\
        &\leq \sqrt{\frac{\iota}{C_dNP}}\sqrt{\sum_{h,s}d_h^{*'}(s)}\sqrt{\sum_{h,s}d_h^{*'}(s)\Var_{P_{h,s,a^*}}(\uv_{h+1}^{\up})}\tag{Cauchy-Schwarz's Inequality}\\
        &\leq \sqrt{\frac{H\iota}{C_dNP}}\sqrt{4H\sum_{h,s}d^{*'}_{h}(s)b_{h}(s,a^*) +2H^2}\tag{Lemma \ref{lemma:variance sum}}\\
        &\leq \sqrt{\frac{4H^2\iota}{C_dNP}\sum_{h,s}d^{*'}_{h}(s)b_{h}(s,a^*)} + \sqrt{\frac{2H^3\iota}{C_dNP}}\tag{$\sqrt{a+b}\leq \sqrt{a}+\sqrt{b}$}\\
        &\leq \frac{4C_bH^2\iota}{C_dNP}+\frac{1}{2C_b}\sum_{h,s}d^{*'}_{h}(s)b_{h}(s,a^*) + \sqrt{\frac{2H^3\iota}{C_dNP}}.\tag{$\sqrt{2ab}\leq a+b$}
    \end{align*}
    At the same time, we can limit the sum of $\frac{H\iota}{N_h'(s,a^*)}$,
    \begin{equation*}
        \sum_{h,s}d_h^{*'}(s)\frac{H\iota}{N_h'(s,a^*)}\leq \sum_{h,s}d_h^{*'}(s)\frac{H\iota}{C_dNP}= \frac{H^2\iota}{C_dNP}.
    \end{equation*}
    By connecting these inequalities to \eqref{state: bh ele treat}, we get
    \begin{align*}
        &\sum_{h,s}d_h^{*'}(s)b_h(s,a^*)\leq C_b\sum_{h,s}2d_h^{*'}(s)\sqrt{\frac{\Var_{P_{h,s,a}}(\uv_{h+1}^{\up})\iota}{N_h'(s,a^*)}}+C_b\sum_{h,s}d_h^{*'}(s)\frac{3H\iota}{N_h'(s,a^*)}\\
        &\leq \frac{1}{2}\sum_{h,s}d_h^{*'}(s)b_h(s,a^*)+(8C_b^2+6C_b)\frac{H^2\iota}{C_dNP}+2C_b\sqrt{\frac{H^3\iota}{C_dNP}}.
    \end{align*}
    Rearranging the terms, we finish the proof of Lemma \ref{lemma:part decomposition}.
\end{proof}
\subsection{Proof of Upper bound with both assumptions (Proof of Theorem \ref{theorem:ub pc})}
\label{ssection:case with both assumptions}
When we have access to both $P$ and $C^*$, we can derive the bound
\begin{equation*}
    N = O\left(\frac{H^3SC^*\iota}{\epsilon\gap}+\frac{H\iota}{P}\right).
\end{equation*}
To prove this, we need a specially designed $\epsilon_{h}(s)$ in Theorem \ref{theorem:main theorem}. By setting 
$\epsilon_{h}(s)=C_{pac}(\frac{\Var_{\hat{P}_{h,s,a^*}}(\uv_{h+1}^{\up})}{H^2}+\frac{1}{H})\gap$, 
we will first prove that $\xi_{gap}$ happenes, and then calculate the suboptimality gap.
\subsubsection{Tools for the Proof of Theorem \ref{theorem:ub pc}}
\label{subsubsection: tools for the proof of theorem ub pc}
\begin{lemma}
\label{lemma:PC limited difference}
If we set $\epsilon_{h}(s,a)=C_{pac}\left(\frac{\Var_{\hat{P}_{h,s,a}}(\uv_{h+1}^{\up})}{H^2}+\frac{1}{H}\right)\gap$
for some small enough constant $C_{pac}$, and $N\geq C_3\frac{H\iota}{P}$ for some constant $C_3$, $\xi_{\mathrm{conc}}$ indicates $\xi_{gap}$, 
    \begin{equation*}
        \forall (h,s)\in[H]\times \mathcal{S}\qquad \ddot{\mathbf{V}}_h^*(s)\leq \uv^*_h(s) + \frac{\gap}{2}.
    \end{equation*}
\end{lemma}
\begin{proof}
    We have proved in the proof of Corollary \ref{corollary:upper bound 2} that 
    \begin{equation*}
        \sum_{h,s}d_h^{*'}(s)\frac{\gap}{4H}\leq \frac{\gap}{4},
    \end{equation*}
    where $d_h^{*'}=d_{h\sim (h',s')}^*$, which is the state distribution of time step $h$ under $\pi^*$ conditioned on having reached $(h',s')$ before.
    So it remains to show that 
    \begin{equation}
        \sum_{h,s}d_h^{*'}(s)\frac{\Var_{\hat{P}}(\uv^{\up_{h+1}})}{H^2}\gap\lesssim \gap. \label{state: PC part sum}
    \end{equation}
    This follows from a similar analysis with the proof of uniform optimal policy coverage assumption case.
    \begin{align*}
        &\sum_{h,s}d_h^{*'}(s)\Var_{\hat{P}}(\uv_{h+1}^{\up})\\
        &\leq \sum_{h,s}d_h^{*'}(s)\left(\Var_{P}(\uv_{h+1}^{\up})+\frac{5H^2\iota}{3N_h'(s,a^*)}\right)\tag{Lemma \ref{lemma:tool for part decomposition}}\\
        &\lesssim \sum_{h,s}d_h^{*'}(s)\Var_{P}(\uv_{h+1}^{\up}) + H^2 \tag{$N_h'(s,a^*)\geq C_d NP\gtrsim H\iota$} \\
        &\leq 4H\sum_{h,s}d_h^{*'}(s)b_h(s,a^*) + 3H^2 \tag{Lemma \ref{lemma:variance sum}}\\
        &\lesssim 4H\left(\sqrt{\frac{H^3\iota}{NP}}+\frac{H^2\iota}{NP}\right) + 3H^2 \tag{Lemma \ref{lemma:part decomposition}}\\
        &\lesssim H^2. \tag{$NP\gtrsim H\iota$}
    \end{align*}
    Then we can finish the proof 
    \begin{align*}
        \ddot{\mathbf{V}}_h^*(s)-\uv_h^*(s)&=\sum_{h,s}d_h^{*'}(s)(\ddot{\mathbf{E}}_h(s,a^*)-\mathbf{E}_h(s,a^*))\\
        &\leq \sum_{h,s}d_h^{*'}(s)\epsilon_{h}(s)\\
        &\lesssim C_{cap}\sum_{h,s}d_h^{*'}\left(\frac{\Var_{\hat{P}}(\uv_{h+1}^{\up})}{H^2}+\frac{1}{H}\right)\gap\\
        &\lesssim C_{cap}\gap.
    \end{align*}
    We can let $C_{cap}$ be small enough to limit the difference between $\ddot{\mathbf{V}}_h^*(s)$ and $\uv_h^*(s)$ within $\frac{\gap}{2}$.
\end{proof}
\begin{lemma}
\label{lemma: PC gap trick}
    \begin{equation*}
        \ddot{\mathbf{E}}_h(s,a^*)\leq 4C_b\sqrt{\frac{\Var_{\hat{P}}(\uv^{\up}_{h+1})\iota}{N_h'(s,a^*)}}\left(\frac{b_h(s,a^*)}{\epsilon_{h}(s)}\right) + 2C_b\frac{H\iota}{N_h(s,a^*)}.
    \end{equation*}
\end{lemma}
\begin{proof}
    When $\mathbf{E}_h(s,a^*)< \epsilon_{h}(s)$, $\ddot{\mathbf{E}}_h(s,a^*)=0$. \\
    When $\mathbf{E}_h(s,a^*)\geq  \epsilon_{h}(s)$,
    \begin{align*}
        2b_h(s,a^*)&\geq \mathbf{E}_h(s,a^*)\geq \epsilon_{h}(s).\\
        \Rightarrow \ddot{\mathbf{E}}_h(s,a^*)&\leq \mathbf{E}_h(s,a^*)\\
        &\leq 2b_h(s,a^*)\\
        &=2C_b\sqrt{\frac{\Var_{\hat{P}}(\uv_{h+1}^{\up})\iota}{N_h'(s,a^*)}}+2C_b\frac{H\iota}{N_h'(s,a^*)}\\
        &\leq  4C_b\sqrt{\frac{\Var_{\hat{P}}(\uv^{\up}_{h+1})\iota}{N_h'(s,a^*)}}\left(\frac{b_h(s,a^*)}{\epsilon_{h}(s)}\right) + 2C_b\frac{H\iota}{N_h'(s,a^*)}.
    \end{align*}
\end{proof}
\subsubsection{Main Proof}
\begin{proof}
We treat the first term in the RHS of Lemma \ref{lemma: PC gap trick}. With basic inequality $a+b\geq 2\sqrt{ab}$, we can first lowe bound $\epsilon_h(s,a)$,
\begin{equation}
    \epsilon_h(s,a) =C_{pac}\left(\frac{\Var_{\hat{P}_{h,s,a}}(\uv_{h+1}^{\up})}{H^2}+\frac{1}{H}\right)\gap \gtrsim \max\left\{\frac{\Var_{\hat{P}_{h,s,a}}(\uv_{h+1}^{\up})}{H^2}, \sqrt{\frac{\Var_{\hat{P}_{h,s,a}}(\uv_{h+1}^{\up})}{H^3}}\right\}\gap.
\end{equation}
Therefore we have,
\begin{align*}
    &\sqrt{\frac{\Var_{\hat{P}}(\uv^{\up}_{h+1})\iota}{N_h'(s,a^*)}}\left(\frac{b_h(s,a^*)}{\epsilon_{h}(s,a^*)}\right)\\
    &\lesssim \sqrt{\frac{\Var_{\hat{P}}(\uv^{\up}_{h+1})\iota}{N_h'(s,a^*)}}\left(\frac{\sqrt{\frac{\Var_{\hat{P}}(\uv^{\up}_{h+1})\iota}{N_h'(s,a^*)}}}{\frac{\Var_{\hat{P}}(\uv_{h+1}^{\up})}{H^2}}+\frac{\frac{H\iota}{N_h(s,a)}}{\sqrt{\frac{\Var_{\hat{P}}(\uv^{\up}_{h+1})}{H^3}}}\right)\frac{1}{\gap}\\
    &= \frac{H^2\iota}{N_h'(s,a^*)\gap} + \frac{H^{5/2}\iota^{3/2}}{N_h^{'3/2}(s,a^*)\gap}\\
    &\lesssim \frac{2H^2\iota}{N_h'(s,a^*)\gap}.
\end{align*}
The first inequality is gained by expanding $b_h(s,a^*)$ and $\epsilon_{h}(s,a)$. The second inequality results from the 
inequality that $N_h'(s,a^*)\geq C_d NP\gtrsim H\iota$. Therefore, we can further write Lemma \ref{lemma: PC gap trick} as 
\begin{equation*}
    \ddot{\mathbf{E}}_h(s,a^*)\lesssim \frac{H^2\iota}{N_h'(s,a^*)\gap}.
\end{equation*}
Then with Lemma \ref{lemma:PC limited difference}, we have event $\xi_{gap}$ hold.
Then Theorem \ref{theorem:main theorem} further indicates that for some deterministic optimal policy $\pi^*$,
\begin{align*}
    \mathbf{V}_0^* - \mathbf{V}_0^{\up}&\lesssim \mathbf{V}_0^*-\ddot{\mathbf{V}}_0^*\\
    &= \sum_{h,s}d_h^*(s)\ddot{\mathbf{E}}_h(s,a^*)\\
    &\lesssim \sum_{h,s}\frac{d_h^*(s)H^2\iota}{N_h'(s,a^*)\gap}\\
    &\lesssim \sum_{h,s}\frac{H^2C^*\iota}{N\gap}\\
    &=\frac{H^3C^*S\iota}{N\gap}\lesssim \epsilon.
\end{align*}
\end{proof}
\section{Gap-dependent Lower Bounds}
We begin by restating the formal version of lower bounds.
\add{
\begin{definition}[offline learning algorithm]
    For an algorithm $\mathbf{ALG}$, we call it an offline learning algorithm if
    \begin{enumerate}
        \item $\mathbf{ALG}$ takes a dataset $\mathcal{D}$ and optionally a reward function $R$ as input,
        \item $\mathbf{ALG}$ output a valid policy $\pi$.
    \end{enumerate}
\end{definition}
Notice that $\mathbf{ALG}$ can be stochastic. 
}
\subsection{Main Results}
\label{appendix:lb results}
\begin{theorem}
    \label{theorem: main lb the}
    There exists constant $C_{lb}$, s.t. for any $A\geq 3,S\geq 2,H\geq 2,\tau<\frac{1}{2},\lambda<\frac{1}{3},\add{\lambda_1\geq 2}$ and algorithm $\mathbf{ALG}$, 
    if the number of sample trajectories 
    \begin{equation*}
        N\leq C_{lb}\cdot\frac{HS\modify{\lambda_1}}{\lambda\tau^2},
    \end{equation*}
    there exists some MDP $\mathcal{M}$ and behavior policy $\mu$ with $\gap=\tau$, $P\geq \frac{\lambda}{eS\modify{\lambda_1}}$, $C^*\leq \modify{\lambda_1}$ such that 
    the output policy $\hat{\pi}$ suffers from a expected suboptimality
    \begin{equation*}
        \mathbb{E}_{\mathcal{M},\mu,\mathbf{ALG}}[\mathbf{V}^*_0-\mathbf{V}^{\hat{\pi}}_0]\geq \frac{\lambda H\tau}{12}.
    \end{equation*}
\end{theorem}
\begin{corollary}[lower bound for uniform optimal policy coverage]
    \label{theorem:lb uopc}
    Given $\add{A\geq 3, S\geq 2}, H\geq 3, P\add{\in(0,\frac{1}{6S}]},\epsilon<1/12, \gap\in [\frac{24\epsilon}{H},\frac{1}{2}]$ and any offline learning algorithm $\mathbf{ALG}$ returning a policy $\hat{\pi}$, there exists a constant $C_1$ such that
     if the number of offline sample trajectories
    \begin{equation*}
        N\leq C_1\cdot \frac{H}{P\gap^2},
    \end{equation*}
    then there exists a MDP instance $M$ and behavior policy $\mu$ such that the output policy $\hat{\pi}$ suffers from expected $\epsilon$-suboptimality
    \begin{equation*}
        \mathbb{E}_{\mathcal{M},\mu,\mathbf{ALG}}[\mathbf{V}_{0}^*-\mathbf{V}_{0}^{\hat{\pi}}]\geq \epsilon.
    \end{equation*}
\end{corollary}
\begin{proof}
    Let $\lambda=\modify{1/3}$, \modify{$\lambda_1=\frac{1}{3PS}$ and
    $\tau=\mathrm{gap}_{\min}$} in Theorem \ref{theorem: main lb the}, we get the proposition.
\end{proof}
\begin{corollary}[lower bound for relative optimal policy coverage]
    \label{theorem:lb ropc}
    Given $\add{A\geq 3}, S\geq 2,H\geq 2,C^*\geq 2,\epsilon<1/12, \gap\in [\frac{24\epsilon}{H},\frac{1}{2}]$ and any offline learning algorithm $\mathbf{ALG}$ returning a policy $\hat{\pi}$, there exists a constant $C_2$ such that
     if the number of offline sample trajectories
    \begin{equation*}
        N\leq C_2\cdot \frac{H^2SC^*}{\gap\epsilon},
    \end{equation*}
    then there exists a MDP instance $\mathcal{M}$ and behavior policy $\mu$ such that the output policy $\hat{\pi}$ suffers from expected $\epsilon$-suboptimality
    \begin{equation*}
        \mathbb{E}_{\mathcal{M},\mu,\mathbf{ALG}}[\mathbf{V}_{0}^*-\mathbf{V}_{0}^{\hat{\pi}}]\geq \epsilon.
    \end{equation*}
\end{corollary}
\begin{proof}
    Let $\lambda=\frac{12\epsilon}{H\gap}$, \modify{$\lambda_1=C^*$ and $\tau=\mathrm{gap}_{\min}$} in Theorem \ref{theorem: main lb the}, we get the conclusion.
\end{proof}
\subsection{Proof of Theorem \ref{theorem: main lb the}}
\subsubsection{Construction of the MDP Family and Behavior Policy}
We construct a MDP family and calculate the average minimum suboptimality.\\
First, we construct the prototype MDP $\mathcal{M}_{0}$ with $S+2$ states, horizon of $2H+1$ and $A$ actions. 
There are 3 kind of states 
\begin{enumerate}
    \item good state $s_g$. An absorbing state. Reaching this state means a total reward of $H$.
    \item bad state $s_b$. An absorbing state. Reaching this state means a total reward of $0$.
    \item true states $s_1,s_2,\cdots, s_S$. Actions chosen in these states determine the probability being transfered to $s_g$ and $s_b$.
\end{enumerate}
The initial state distribution $p_0(s)$ is 
\begin{equation*}
    p_0(s)=\begin{cases}
        \frac{\lambda}{S}&s\in\{s_1,s_2,\cdots,s_S\},\\
        \frac{1-\lambda}{2} &s=s_b,\\
        \frac{1-\lambda}{2} &s=s_g.
    \end{cases}
\end{equation*}
For any $\lambda\in[0,\frac{1}{3}]$.
The avaliable action set is $\{a_i\}_{i=1}^A$.
The only non-zero rewards in this MDP are $r_h(s_g,a)=1$ for $h\geq H+2$ and any $a$. 
The transition probability of $\mathcal{M}_0$ in the first $H+1$ steps is,
\begin{align*}
    &p_h(s_i,a_j,s_i) = 1-\frac{1}{H} &\forall (h,i,j)\in[H]\times[S]\times[A],\\
    &p_h(s_i,a_j,s_g) = p_h(s_i,a_j,s_b) = \frac{1}{2H} &\forall (h,i,j)\in[H]\times[S]\times[A],\\
    &p_{H+1}(s_i,a_j,s_g)=p_{H+1}(s_i,a_j,s_b)=\frac{1}{2} &\forall (i,j)\in[S]\times[A].
\end{align*}
For all the other $(h,s,a)$ tuples not mentioned, $p_h(s,a,s)=1$, $p_h(s,a,s')=0$, where $s'$ is any state other than $s$.\\
Then we construct the MDP family $M$ on the basis of $\mathcal{M}_0$. \modify{For each matrix $\phi\in [1,2]^{H\times S}$}, we define $\mathcal{M}_\phi$ 
to be the MDP almost the same as $\mathcal{M}_0$ except for that 
\begin{align*}
    &p_h(s_i,a_{\phi_{h,i}},s_g) = \frac{1}{2H}(1+2\tau),\\
    &p_h(s_i,a_{\phi_{h,i}},s_b) = \frac{1}{2H}(1-2\tau).
\end{align*}
In other words, we make the action $a_{\Phi_{h,i}}$ the unique optimal action by lifting it's expected reward by $\tau$. \modify{The behavior
policy $\mu$ chooses $a_1, a_2$ with probability $1/\lambda_1$ respectively and choose $a_3$ with probability $1-2/\lambda_1$ at $\{s_i\mid i=1,2,\cdots S\}$, and always 
choose $a_1$ at $s_g$ and $s_b$. }
We will prove the following lemmas in Section \ref{sssection: verification of the assumptions} 
\begin{lemma}
\label{lemma: lb con ass}
    For any MDP constructed above and $\mu$, we have both assumptions hold with 
    \begin{equation*}
        C^*=\modify{\lambda_1}, 
        P\geq \frac{\lambda}{eS\modify{\lambda_1}},
        \gap = \tau.
    \end{equation*}   
\end{lemma}
\begin{lemma}
    \label{lemma: lb lose sum}
    For a given algorithm $\mathbf{ALG}$, define the expectation of mistakes made by $\hat{\pi}$ at step $h$, state $s_i$ over the uniform distribution $\nu$ of $\phi$ to be 
    \begin{equation*}
        l_{h,i}(\mathbf{ALG})=\mathbb{E}_{\mathcal{D}\sim(\mathcal{M}_{\phi},\mu),\mathbf{ALG}}[\mathbb{I}\{\hat{\pi}_h(s_i)\neq a_{\phi_{h,i}}\}].
    \end{equation*}
    Then expected suboptimality with respect to the randomness of $\mathcal{M}_{\phi}$ and $\mu$ can be lower bounded by 
    \begin{equation*}
        \mathbb{E}_{\mathcal{M}_{\phi},\mu,\mathbf{ALG}}[\mathbf{V}_0^{*}-\mathbf{V}_0^{\hat{\pi}}]\geq \frac{\lambda}{eS}\tau\sum_{(h,i)\in[H]\times[S]}l_{h,i}(\mathbf{ALG}).
    \end{equation*}
\end{lemma}
\add{
\begin{lemma}
\label{lemma: maximum suboptimality}
For any MDP constructed above, we have,
\begin{equation*}
    \underset{\pi}{\max}\ \mathbf{V}_0^*-\mathbf{V}_0^\pi\leq \lambda H\tau.
\end{equation*}
\end{lemma}
}
\subsubsection{Main Proof}
\begin{proof}
    To avoid making the proof prolix, we strengthen $\mathbf{ALG}$ by letting $\mathbf{ALG}$ know that the only thing influencing the value function of a state is the probabilities of transferring 
    to $s_g$ and $s_b$, and the total reward after getting to $s_g$ is exactly $H$, which assumption is conventionally made for the lower bound proofs in MDPs. In this setting, any reasonable algorithm $\mathbf{ALG}$ would only consider the visitation counts 
    $N_{h,i}=\{N_h(s_i,a,s')\mid a\in\mathcal{A}, s'\in\mathcal{S}\}$ at step $h$ when determining the value of $\hat{\pi}_h(s_i)$. 
    \\
    Then we can rewrite $\bar{l}_{h,i}(\mathbf{ALG})\triangleq\mathbb{E}_{\phi\sim\nu}[l_{h,i}(\mathbf{ALG})]$,
    \begin{align*}
        \bar{l}_{h,i}(\mathbf{ALG})&=\mathbb{E}_{\phi\sim\nu}\mathbb{E}_{\mathcal{D}\sim(\mathcal{M}_{\phi},\mu),\mathbf{ALG}}[\mathbb{I}\{\hat{\pi}_h(s_i)\neq a_{\phi_{h,i}}\}]\\
        &=\mathbb{E}_{\phi\sim\nu}\mathbb{E}_{N_{h,i}\sim(\mathcal{M}_{\phi},\mu),\mathbf{ALG}}[\mathbb{I}\{\hat{\pi}_h(s_i)\neq a_{\phi_{h,i}}\}].
    \end{align*}
    Because the KL divergence between the transition kernel $P^{\mathcal{M}_0}_{h,s_i,a}$ and $P^{\mathcal{M}_\phi}_{h,s_i,a_{\phi_{h,i}}}$ satisfies
    \begin{align}
        &KL\left(\left(1-\frac{1}{H},\frac{1}{2H},\frac{1}{2H}\right)\Big\Arrowvert\left(1-\frac{1}{H},\frac{1}{H}(\frac{1}{2}+2\tau),\frac{1}{H}(\frac{1}{2}-2\tau)\right)\right)\notag\\
        &=\frac{1}{2H}\log \frac{1}{1-4\tau^2}\leq \frac{4\tau^2}{H}. \label{med con: lb kl}
    \end{align}
    We have 
    \begin{align*}
        \bar{l}_{h,i}(\mathbf{ALG})&=\mathbb{E}_{\phi\sim\nu}\mathbb{E}_{N_{h,i}\sim(\mathcal{M}_{\phi},\mu),\mathbf{ALG}}[\mathbb{I}\{\hat{\pi}_h(s_i)\neq a_{\phi_{h,i}}\}]\\
        &\geq \mathbb{E}_{\phi\sim\nu}\mathbb{E}_{N_{h,i}\sim(\mathcal{M}_{0},\mu),\mathbf{ALG}}[\mathbb{I}\{\hat{\pi}_h(s_i)\neq a_{\phi_{h,i}}\}]\\
        &\qquad -\mathbb{E}_{\phi\sim\nu}[\text{TV}(N_{h,i}\mid_{\mathcal{M}_0,\mu},N_{h,i}\mid_{\mathcal{M}_\phi,\mu})]\\
        &\geq \mathbb{E}_{N_{h,i}\sim(\mathcal{M}_{0},\mu),\mathbf{ALG}}\mathbb{E}_{\phi\sim\nu}[\mathbb{I}\{\hat{\pi}_h(s_i)\neq a_{\phi_{h,i}}\}]\\
        &\qquad - \mathbb{E}_{\phi\sim\nu}\sqrt{\frac{1}{2}\text{KL}(N_{h,i}\mid_{\mathcal{M}_0,\mu}\Arrowvert N_{h,i}\mid_{\mathcal{M}_\phi,\mu})}\tag{Pinsker's inequality}\\
        &\modify{\geq \mathbb{E}_{N_{h,i}\sim(\mathcal{M}_{0},\mu)}[\frac{1}{2}]} \tag{\modify{$a_1$ and $a_2$ can not be distinguished in $\mathcal{M}_0$}} \\
        &\qquad - \mathbb{E}_{\phi\sim\nu}\sqrt{\frac{1}{2}\modify{\sum_{a\in \mathcal{A}}}\mathbb{E}_{\mathcal{M}_0,\mu}[N_h(s_i,a)]\text{KL}(P^{\mathcal{M}_0}_{h,s_i,a}\Arrowvert P^{\mathcal{M}_\phi}_{h,s_i,a})} \tag{KL decomposition}\\
        &=\modify{\frac{1}{2}} - \mathbb{E}_{\phi\sim\nu}\sqrt{\frac{1}{2}\mathbb{E}_{\mathcal{M}_0,\mu}[N_h(s_i,a)]\text{KL}(P^{\mathcal{M}_0}_{h,s_i,a_{\phi_{h,i}}}\Arrowvert P^{\mathcal{M}_\phi}_{h,s_i,a_{\phi_{h,i}}})} \\
        &\geq \frac{1}{2} - \mathbb{E}_{\phi\sim\nu}\sqrt{\mathbb{E}_{\mathcal{M}_0,\mu}[N_h(s_i,a_{\phi_{h,i}})]\frac{2\tau^2}{H}}  \tag{statement \eqref{med con: lb kl}}\\
        &\geq \frac{1}{2} - \sqrt{\frac{1}{\modify{2}}\modify{\sum_{a=a_1,a_2}}\mathbb{E}_{\mathcal{M}_0,\mu}[N_h(s_i,a)]\frac{2\tau^2}{H}}  \tag{Jensen's inequality}\\
        &=\frac{1}{2} - \sqrt{\mathbb{E}_{\mathcal{M}_0,\mu}[N_h(s_i)]\frac{2\tau^2}{H\modify{\lambda_1}}}. \tag{\modify{$\frac{\sum_{a=a_1,a_2}\mathbb{E}_{\mathcal{M}_0,\mu}[N_h(s_i,a)]}{\mathbb{E}_{\mathcal{M}_0,\mu}[N_h(s_i)]}=\mu_{h}(a_1\mid s_i)+\mu_{h}(a_2\mid s_i)=\frac{2}{\lambda_1}$}}
    \end{align*}
This further indicates that the expectation of overall mistakes can be lower bounded by 
\begin{align}
    \sum_{(h,i)\in[H]\times[S]}\bar{l}_{h,i}(\mathbf{ALG}) &\geq \sum_{h,i}\left(\frac{1}{2} - \sqrt{\mathbb{E}_{\mathcal{M}_0,\mu}[N_h(s_i)]\frac{2\tau^2}{H\modify{\lambda_1}}}\right)\notag\\
    &\geq \frac{HS}{2}- \sqrt{HS}\sqrt{\sum_{h,i}\mathbb{E}_{\mathcal{M}_0,\mu}[N_h(s_i)]\frac{2\tau^2}{H\modify{\lambda_1}}} \tag{Cauchy Schwarz's Inequality}\notag\\
    &=HS(\frac{1}{2}-\sqrt{\mathbb{E}_{\mathcal{M}_0,\mu}[\sum_{h,i}N_h(s_i)]\frac{2\tau^2}{H^2S\modify{\lambda_1}}}).\label{state: lb aft csi}
\end{align}
Because state $s_i$ can be reached at step $h$ only when the initial state is $s_i$,
\begin{align*}
    \mathbb{E}_{\mathcal{M}_0,\mu}[\sum_{(h,i)\in[H]\times[S]}N_h(s_i)]=&N\sum_{h,i}\frac{\lambda}{S}(1-\frac{1}{H})^{h-1}\\
    &\leq N\sum_{h,i}\frac{\lambda}{S} = NH\lambda.
\end{align*}
Therefore, continue from inequality \eqref{state: lb aft csi},
\begin{equation*}
    \sum_{(h,i)\in[H]\times[S]}\bar{l}_{h,i}(\mathbf{ALG})\geq HS\left(\frac{1}{2}-\sqrt{N\frac{2\lambda\tau^2}{HS\modify{\lambda_1}}}\right),
\end{equation*}
Now we can lower bound the suboptimality of $\hat{\pi}$ with Lemma \ref{lemma: lb lose sum} by
\begin{equation*}
    \mathbb{E}_{\phi\sim\nu}\mathbb{E}_{\mathcal{M}_{\phi},\mu,\mathbf{ALG}}[\mathbf{V}_0^*-\mathbf{V}^{\hat{\pi}}_0]\geq \mathbb{E}_{\phi\sim\nu}[\frac{\lambda\tau}{eS}\sum_{h,i}l_{h,i}(\mathbf{ALG})]\geq \frac{\lambda H\tau}{e}  \left(\frac{1}{2}-\sqrt{N\frac{2\lambda\tau^2}{HS\modify{\lambda_1}}}\right).
\end{equation*}
Then we reach the conclusion that when 
\begin{equation*}
    N\leq \frac{HS\modify{\lambda_1}}{32\lambda\tau^2},
\end{equation*}
the average suboptimality of $\hat{\pi}$ must be large 
\begin{align*}
&\mathbb{E}_{\phi\sim\nu}\mathbb{E}_{\mathcal{M}_{\phi},\mu,\mathbf{ALG}}[\mathbf{V}_0^*-\mathbf{V}^{\hat{\pi}}_0]\geq \frac{\lambda H\tau}{e}\cdot (\frac{1}{2}-\frac{1}{4}),\\
    \Rightarrow&\exists \phi, s.t. \mathbb{E}_{\mathcal{M}_{\phi},\mu,\mathbf{ALG}}[\mathbf{V}_0^*-\mathbf{V}^{\hat{\pi}}_0]\geq \frac{\lambda H \tau}{4e}.
\end{align*}
\end{proof}
\add{
\subsubsection{Constant probability version of main theorem}
\label{sssection: prob lcb theo}
Theorem \ref{theorem: main lb the} is stated in the form of expectation, which is not directly consist with upper bound. Here we restate it in the language of probability,
\begin{theorem}
    \label{theorem: main lb the prob}
    There exists constant $C_{lb}$, s.t. for any $A\geq 3,S\geq 2,H\geq 2,\tau<\frac{1}{2},\lambda<\frac{1}{3},\lambda_1\geq 2$ and algorithm $\mathbf{ALG}$, 
    if the number of sample trajectories 
    \begin{equation*}
        N\leq C_{lb}\cdot\frac{HS\lambda_1}{\lambda\tau^2},
    \end{equation*}
    there exists some MDP $\mathcal{M}$ and behavior policy $\mu$ with $\gap=\tau$, $P\geq \frac{\lambda}{eS\lambda_1}$, $C^*\leq \lambda_1$ such that 
    the output policy $\hat{\pi}$ suffers from a expected suboptimality
    \begin{equation*}
        \mathbf{V}^*_0-\mathbf{V}^{\hat{\pi}}_0\geq \frac{\lambda H\tau}{24},
    \end{equation*}
    with a probability over $\frac{1}{24}$.
\end{theorem}
\begin{proof}
    From the last line of the proof of Theorem \ref{theorem: main lb the}, we know that there exists a MDP $\mathcal{M}$, such that 
    \begin{equation*}
        \mathbb{E}_{\mathcal{M}_{\phi},\mu,\mathbf{ALG}}[\mathbf{V}_0^*-\mathbf{V}^{\hat{\pi}}_0]\geq \frac{\lambda H \tau}{12},
    \end{equation*}
    and it follows from Lemma \ref{lemma: maximum suboptimality} that the random variable $\mathbf{V}_0^*-\mathbf{V}^{\hat{\pi}}_0\leq \lambda H \tau$. Therefore 
    \begin{align*}
        \frac{\lambda H \tau}{12} &\leq \mathbb{E}_{\mathcal{M}_{\phi},\mu,\mathbf{ALG}}[\mathbf{V}_0^*-\mathbf{V}^{\hat{\pi}}_0]\\
        &\leq \lambda H \tau\mathbb{P}[\mathbf{V}_0^*-\mathbf{V}^{\hat{\pi}}_0 > \frac{\lambda H \tau}{24}] + \frac{\lambda H \tau}{24}\mathbb{P}[\mathbf{V}_0^*-\mathbf{V}^{\hat{\pi}}_0 \leq \frac{\lambda H \tau}{24}]\\
        &\leq \frac{\lambda H\tau}{12}(12\mathbb{P}[\mathbf{V}_0^*-\mathbf{V}^{\hat{\pi}}_0 > \frac{\lambda H \tau}{24}]+\frac{1}{2})\\
        &\Rightarrow \mathbb{P}[\mathbf{V}_0^*-\mathbf{V}^{\hat{\pi}}_0 > \frac{\lambda H \tau}{24}]\geq \frac{1}{24}
    \end{align*}
\end{proof}
}

\subsubsection{Proof of Lemma \ref{lemma: lb con ass}}
\label{sssection: verification of the assumptions}
\begin{proof}
From the construction we see that the policy doesn't influence the probability of reaching $s_i$ at any time step. And the uniform random 
behavior policy makes sure that there is a chance of $1/A$ to visit action $a_i$ at any state. For $s_g$ and $s_b$, because each of them 
has a initial probability of $1/3$, the probability of reaching one of them at ant time step would be in $[\frac{1}{3},\frac{2}{3}]$. By 
letting a optimal policy always choose $a_1$ in $s_g$ and $s_b$ as $\mu$ does, we make 
$\frac{d_h^*(s_g)}{d_h^\mu(s_g)}\leq \frac{2/3}{1/3}=2\leq C^*$. Therefore $C^*=\modify{\lambda_1}$\\
As for P, we see that the probabality to reach $s_g$ and $s_d$ with behavior policy at step $h$ is 
\begin{align*}
    &d_h^\mu(s_g,a)\geq \modify{\frac{1-\lambda}{2}\geq \frac{\lambda}{2}\geq \frac{\lambda}{eS\lambda_1}},\\
    &d_h^\mu(s_b,a)\geq \modify{\frac{1-\lambda}{2}\geq \frac{\lambda}{2}\geq \frac{\lambda}{eS\lambda_1}},\\
    &d_h^\mu(s_i,a)=\modify{\frac{\lambda}{S}(1-\frac{1}{H})^{h-1}\frac{1}{\lambda_1}\geq \frac{\lambda}{eS\lambda_1}}.
\end{align*}
The part for $\gap$ is direct calculation,
\begin{equation*}
    \gap = \frac{1}{2H}(1+2\tau)H+ 0 - \frac{1}{2H}H- 0= \tau.
\end{equation*}
\end{proof}
\subsubsection{Proof of Lemma \ref{lemma: lb lose sum}}
\begin{proof}
Because $\hat{\pi}$ only makes mistakes in $s_i$, and each mistake results in a expected $\tau$ decrease in final cumulative reward, 
we can directly calculate the expected loss with the performance difference lemma for finite-horizon MDP ,
\begin{align*}
    \mathbb{E}_{\mathcal{M}_{\phi},\mu,\mathbf{ALG}}[\mathbf{V}_0^*-\mathbf{V}_0^{\hat{\pi}}]&=\mathbb{E}_{\hat{\pi},\mathcal{M}_{\phi}}[\mathrm{gap}_h(s_h,a_h)]\\
    &=\sum_{i=1}^S\sum_{h=1}^H d_h^\mu(s_i)l_{h,i}(\mathbf{ALG})\tau\\
    &=\sum_{i=1}^S\sum_{h=1}^HP_0(s_i)(1-\frac{1}{H})^{h-1}\tau l_{h,i}(\mathbf{ALG})\\
    &\geq \sum_{i=1}^S\sum_{h=1}^H\frac{\lambda}{eS}\tau l_{h,i}(\mathbf{ALG})\\
    &\geq \frac{\lambda}{eS}\tau\sum_{(h,i)\in[H]\times[S]}l_{h,i}(\mathbf{ALG}).
\end{align*}
\end{proof}

\add{
\subsubsection{Proof of Lemma \ref{lemma: maximum suboptimality}}
\begin{proof}
    The proof is a direct result of performance decomposition lemma.
    \begin{align*}
        \underset{\pi}{\max}\ \mathbf{V}_0^* - \mathbf{V}_0^\pi& = \underset{\pi}{\max}\ \sum_{h=1}^{2H}\mathbb{E}_{\pi^*}[\mathbf{V}_h^*(s) - \mathbf{V}_h^\pi(s)]\\
        &=\underset{\pi}{\max}\ \sum_{h=1}^{H}\mathbb{E}_{\pi^*}[\mathbf{V}_h^*(s) - \mathbf{V}_h^\pi(s)]\\
        &\leq \sum_{h=1}^{H}\mathbb{E}_{\pi^*}[\underset{\pi}{\max}\ \mathbf{V}_h^*(s) - \mathbf{V}_h^\pi(s)]\\
        &=\sum_{h=1}^{H}\sum_{i=1}^Sd_h^*(s)\underset{\pi}{\max}\ \mathbf{V}_h^*(s_i) - \mathbf{V}_h^\pi(s_i)\\
        &=\sum_{h=1}^{H}\sum_{i=1}^Sd_h^*(s)\tau\\
        &=\lambda H \tau
    \end{align*}
\end{proof}
}
\section{Proof of Necessity of Overall Data Coverage}
One may wonder if the Assumption \ref{assumption:uniform optimal policy coverage} has been too strong, because the minimax bound $O(\frac{H^3\iota}{P\epsilon^2})$ only requires the data coverage over a single optimal policy. Here we give a proof that to derive $\epsilon$-irrevelant bounds for Algorithm \ref{alg:LCB with subsampling}, single optimal policy coverage is not sufficient.\\
We provide a hard instance to prove that if we only have data coverage over one of the optimal policies, Algorithm \ref{alg:LCB with subsampling} may output suboptimal policy with probability over $1/2$. We use $P'$ to refer to the single policy coverage coefficient, i.e.,
\begin{equation*}
    P' = \underset{\pi^*}{\max}\underset{d_h^{\pi^*}(s,a)>0}{\min}d_h^{\mu}(s,a).
\end{equation*}

Here we consider the MDP with horizon length 2, 2 actions and k+2 states, which is illustrated in Figure \ref{fig:hardinstance}.
\begin{figure}
    \centering
    \includegraphics[width=.6\textwidth]{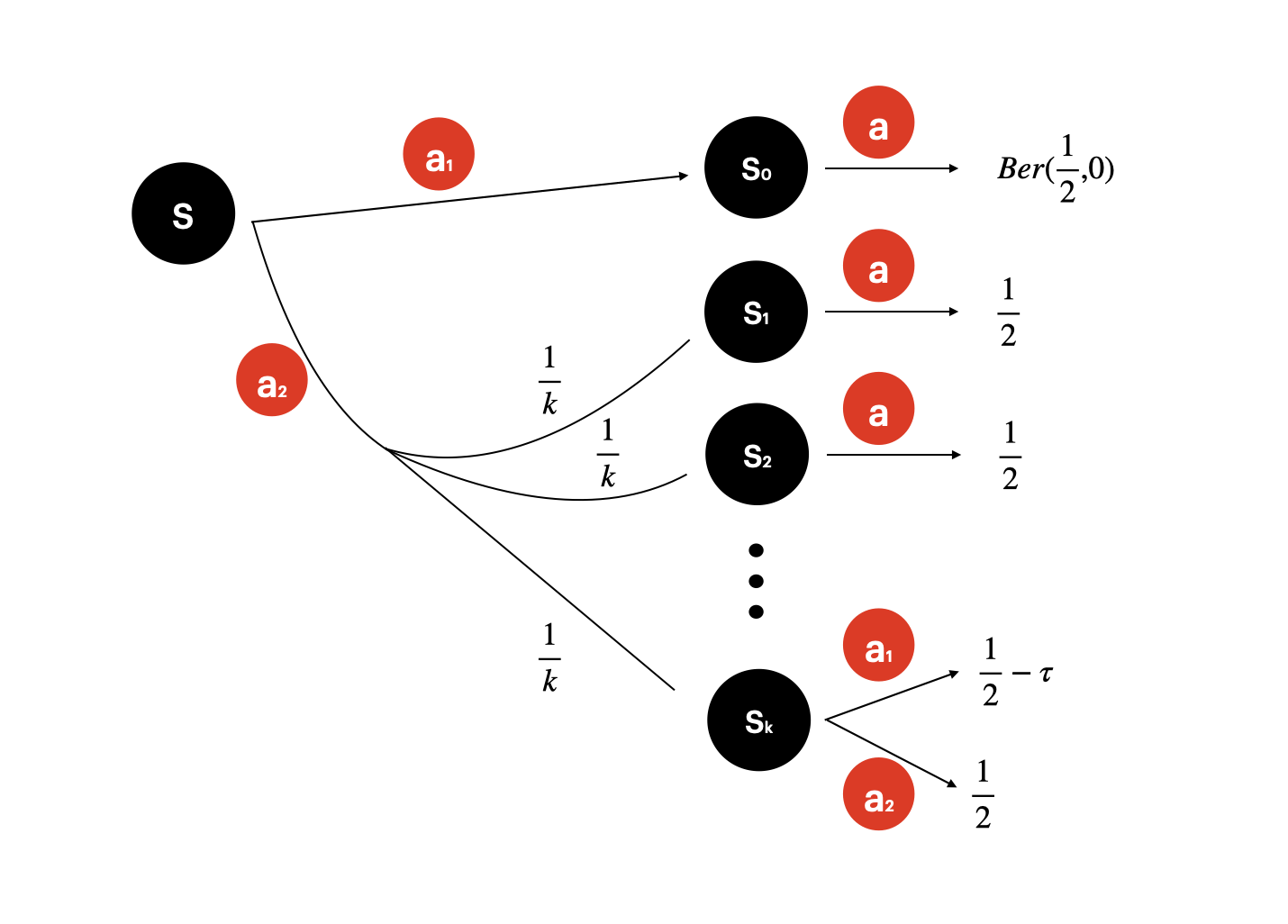}
    \caption{A hard instance with horizon 2, 2 actions and $k+2$ states. $a_1$ at $s$ leads to $s_0$. The reward of both actions at $s_0$ is sampled from Bernoulli Distribution. $a_2$ at $s$ leads to a uniformly random transition to $s_i,i=1,\cdots,k$. The reward of both actions at $s_i,i=1,\cdots, k-1$ are $\frac{1}{2}$. $a_1$ at $s_k$ receives $\frac{1}{2}-\tau$ reward and $a_2$ at $s_k$ receives $\frac{1}{2}$ reward. $\gap=\tau$ in this MDP.}
    \label{fig:hardinstance}
\end{figure}
The initial state is $s$, and $P(s,a_1,s_0)=1,P(s,a_2,s_i)=\frac{1}{k}$ for $i=1,2\cdots,k$. The rewards of actions in $s_i,i=1,2\cdots, k-1$ are all $\frac{1}{2}$, and the rewards of both actions in $s_0$ are sampled from $\mathcal{N}(\frac{1}{2}, 1)$. $r(s_k,a_1)=\frac{1}{2}-\tau$, $r(s_k,a_2)=\frac{1}{2}$.

In this MDP $\gap=\epsilon$, and the only suboptimal action is to take $a_2$ at $s_k$. Define $\mu$,
\begin{align*}
    &\mu(a_1\mid s)=\frac{1}{k+1},\\
    &\mu(a_2\mid s) = \frac{k}{k+1},\\
    &\mu(a_1\mid s_i)= 1\qquad i=0,1,\cdots, k,\\
    &\mu(a_2\mid s_i)= 0\qquad i=0,1,\cdots, k.
\end{align*}
Then we can see that an optimal policy/route $s-a_1-s_0-a_1$ has been covered by $\mu$ with minimal coverage distribution $P'=\frac{1}{k+1}$. Then we show that for any constant $C$, the output policy of VI-LCB with $N=\frac{C}{P'\gap^2}=\frac{C(k+1)}{\tau^2}$ sample trajectories cannot be guaranteed to be optimal with high probability. For the conciseness of proof, we assume that $C>10$, $C_b>16$ and let $k>10$.

Intuitively, this is because there exists the probability that some not-so-well covered optimal policy outperforms the covered one in execution process, and no optimality can be guaranteed over the not-so-well covered one. In this instance, $(s,a_2)$ is also optimal, but as no information about $s_k,a_2$ is known by VI-LCB, it will choose $a_1$ following the principle of pessimism.

In the following proof, we omit the subscripts indicating the time step because the each state only appears in time step 1 or 2, which will not incur confusion.
Rigorously, we define the event the $\{\hat{\pi}(s)=a_2\}$ as $\xi_{\mathrm{bad}}$,
\begin{align*}
    \mathbb{P}[\xi_{bad}]&=\mathbb{P}[\uq(s,a_1)\leq\uq(s,a_2)]\\
    &\geq \mathbb{P}[\uq(s,a_1)\leq \frac{1}{2}-\lambda \tau\leq \uq(s,a_2)]\\
    &\geq 1 - \mathbb{P}[\uq(s,a_1) \geq \frac{1}{2}-\lambda\tau] - \mathbb{P}[\uq(s,a_2)\leq \frac{1}{2} - \lambda\tau].
\end{align*}
where $\lambda$ can be any positive constant, which will be determined later. We limit these two terms respectively.
\begin{align*}
    \mathbb{P}[\uq(s,a_1)\geq \frac{1}{2}-\lambda\tau]&\leq \mathbb{P}[N(s,a_1)\geq \frac{C_1N}{k+1}]+ \mathbb{P}[N(s,a_1)\leq \frac{C_2N}{k+1}]\\
    &+ \mathbb{P}[N(s,a_1)\in [\frac{C_2C}{\tau^2},\frac{C_1C}{\tau^2}],\ \uq(s,a_1)\geq \frac{1}{2}-\lambda\tau].
\end{align*}
Because $N(s,a_1)\sim Bio(N,\frac{1}{k+1})$, it follows from the asymptotic feature of binomial distribution that  $N(s,a_1)-\frac{C}{\tau^2}\sim subG(\frac{Ck}{\tau^2(k+1)})$, and then 
\begin{align*}
    &\mathbb{P}[N(s,a_1)\geq \frac{C_1C}{\tau^2}] \leq \exp\left(-\frac{(C_1-1)^2C^2}{2(1-1/(k+1))^2}\right)\leq \exp(-\frac{C^2}{4}(C_1-1)^2),\\
    &\mathbb{P}[N(s,a_1)\leq \frac{C_2C}{\tau^2}] \leq \exp\left(-\frac{(C_2-1)^2C^2}{2(1-1/(k+1))^2}\right)\leq \exp(-\frac{C^2}{4}(C_2-1)^2).
\end{align*}
Let $C_1=1.5, C_2=0.5$. Remember that we assume that $C>10$, and this makes the sum of above two terms a small constant smaller than $0.1$.
Because $\uq(s,a_1)=-b(s,a_1) + \hat{r}(s_0,a_1) - b(s,a_1)\leq -\frac{3C_b\sqrt{\iota}}{\sqrt{N(s,a_1)}} + \hat{r}(s_0,a_1)$, and the center limit theorem allow us to use $X\sim \mathcal{N}(\frac{1}{2}, \frac{1}{N(s_0,a_1)})$ to replace $\hat{h}(s_0,a_1)$,
\begin{align*}
    \mathbb{P}[N(s,a_1)\in [\frac{C_2C}{\tau^2},\frac{C_1C}{\tau^2}],\ \uq(s,a_1)\geq \frac{1}{2}-\lambda\tau]&\leq \mathbb{P}[\hat{r}(s_0,a_1) - \frac{3C_b\sqrt{\iota}\tau}{\sqrt{C_1C}}\geq \frac{1}{2}-\lambda\tau]\\
    & \lesssim \mathbb{P}[X-\frac{1}{2}\geq \frac{3C_b\sqrt{\iota}\tau}{\sqrt{C_1C}}-\lambda\tau]\\
    &\leq \exp(-\frac{1}{2}(\frac{3C_b\sqrt{\iota}}{\sqrt{C_1C}}-\lambda)^2CC_2)\\
    &\leq \exp(-(\frac{1}{2}C_b\sqrt{\iota}-\lambda)^2).
\end{align*}
The above term disappears quickly when $\iota=\Omega(\log\frac{k}{\delta})$ becomes larger. We will choose a $\lambda<\frac{1}{4}C_b\sqrt{\iota}$.
Then we consider $\uq(s,a_2)$,
\begin{align*}
    \uq(s,a_2)&=-b(s,a_2)-\sum_{i=1}^k\hat{P}(s,a_2,s_i)(r(s_i,a_1)-b(s_i,a_1))\\
    &=\frac{1}{2}-\frac{N(s,a_2,s_k)\tau}{N(s,a_2)} - \frac{(k+1)C_b\iota}{N(s,a_2)} - C_b \sqrt{\frac{\Var_{\hat{P}_{s,a_2}}(\uv)\iota}{N(s,a_2)}}.
\end{align*}
With similar induction, we can prove that event $\{N(s,a_2)\geq \frac{Ck}{2\tau^2}\}\bigcap\{N(s,a_2,s_i)\in [\frac{C}{2\tau^2},\frac{2C}{\tau^2}] \}_{i=1,\cdots, k}$ with probability over 0.8. When we have this concentration event true,
\begin{equation*}
    \mathbb{P}[\uq(s,a_2)\leq \frac{1}{2} - \lambda\tau]\leq \mathbb{P}[-\frac{\tau}{4k}-C_b\tau^2\iota - \frac{C_b\tau^2\iota}{k}\leq -\lambda\tau].
\end{equation*}
By taking $\lambda = \frac{1}{4k} + C_b\tau\iota + \frac{C_b\tau^2}{k} + \epsilon$, where $\epsilon$ is a extremely small positive constant, $\mathbb{P}[\uq(s,a_2)\leq \frac{1}{2} - \lambda\tau]=0$. And further letting $\tau\leq \frac{1}{40\sqrt{\iota}}$, we have $\lambda\leq \frac{1}{4}C_b\sqrt{\iota}$. Putting the above inductions together,
\begin{equation*}
    \mathbb{P}[\xi_{bad}]\geq 1-0.1-0.2-\exp(-\frac{1}{16}C_b^2\iota)\geq \frac{1}{2}.
\end{equation*}
This result points out that Algorithm \ref{alg:LCB with subsampling} has a chance of over 1/2 to return a suboptimal policy. Therefore a overall coverage over all the optimal policies is necessary to derive a $\epsilon$-irrelevant bound for VI-LCB.
\end{document}